\documentclass[nohyperref]{article}

\usepackage{microtype}
\usepackage{graphicx}
\usepackage{subfigure}
\usepackage{booktabs} % for professional tables

\usepackage{hyperref}
\usepackage{enumitem}

\usepackage[accepted]{icml2022}

\usepackage{amsmath}
\usepackage{amssymb}
\usepackage{mathtools}
\usepackage{amsthm}
\usepackage{xspace}

\usepackage{comment}
\usepackage{dsfont}
\usepackage{multirow}
\usepackage{graphicx}
\usepackage{wrapfig}
\usepackage{xcolor}
\usepackage{capt-of}
\usepackage[capitalize,noabbrev]{cleveref}
\usepackage{enumitem}
\setlist[itemize]{noitemsep, nolistsep}

\newcommand{\originalgrumbler}[2]{\begin{quote}\textcolor{blue}{\sl{\bf #1 says:} #2}\end{quote}}
\newcommand{\grumbler}[2]{\originalgrumbler{#1}{#2}}
\newcommand{\minjia}[1]{\grumbler{Minjia}{#1}}

\newcommand\later[1]{\begin{quote}\textcolor{green}{\textbackslash \textbf{later\{}} #1 \textcolor{green}{\}}\end{quote}}
\newcommand\notes[1]{\begin{quote}\textcolor{green}{\textbackslash \textbf{notes\{}} #1 \textcolor{green}{\}}\end{quote}}

\newcommand{\sfsmaller}{}

\newcommand{\bb}{BERT$_{base}$}

\newcommand{\name}{ScaLA\xspace}

\renewcommand{\notes}[1]{}
\renewcommand{\later}[1]{}
\renewcommand{\href}[2]{}

\newcommand{\outline}[1]{\grumbler{outline}{#1}}
\renewcommand{\outline}[1]{}

\usepackage[skip=2pt]{caption}
\usepackage{setspace}
\usepackage{enumitem}
\setlist{nosep} % Change space between items in a listing.
\newcommand{\setvspace}[2]{%
  #1 = #2
  \advance #1 by -1\parskip}

\makeatletter % Control space around theorem and lemma
\def\thm@space@setup{%
  \thm@preskip=3pt
  \thm@postskip=\thm@preskip % or whatever, if you don't want them to be equal
}
\makeatother

\makeatletter
\g@addto@macro\normalsize{%
  \setlength\abovedisplayskip{1pt}
  \setlength\belowdisplayskip{1pt}
  \setlength\abovedisplayshortskip{1pt}
  \setlength\belowdisplayshortskip{1pt}
}
\makeatother

\makeatletter % Control the space around the proof.
\renewenvironment{proof}[1][\proofname]{\par
  \vspace{1pt}% remove the space after the theorem
  \pushQED{\qed}%
  \normalfont
  \topsep0pt \partopsep0pt % no space before
  \trivlist
  \item[\hskip\labelsep
        \itshape
    #1\@addpunct{.}]\ignorespaces
}{%
  \popQED\endtrivlist\@endpefalse
  \addvspace{3pt plus 3pt} % some space after
}
\makeatother 

\theoremstyle{plain}
\newtheorem{theorem}{Theorem}[section]

\newtheorem{lemma}[theorem]{Lemma}

\theoremstyle{definition}

\theoremstyle{remark}

\usepackage[textsize=tiny]{todonotes}

\icmltitlerunning{ScaLA: Accelerating Adaptation of Pre-Trained Transformer-Based Language Models}

\begin{document}

\twocolumn[
\icmltitle{ScaLA: Accelerating Adaptation of Pre-Trained Transformer-Based Language Models via Efficient Large-Batch Adversarial Noise}
% \icmltitle{ScaLA: Speeding-Up Fine-tuning of Pre-trained Transformer Networks via Efficient and Scalable Adversarial Perturbation}



\icmlsetsymbol{equal}{*}

\begin{icmlauthorlist}
\icmlauthor{Minjia Zhang}{comp}
\icmlauthor{Niranjan Uma Naresh}{comp}
\icmlauthor{Yuxiong He}{comp}
\end{icmlauthorlist}

\icmlaffiliation{comp}{Microsoft, Bellevue, WA, USA}

\icmlcorrespondingauthor{Minjia Zhang}{minjiaz@microsoft.com}
\icmlcorrespondingauthor{Niranjan Uma Naresh}{Niranjan.Uma@microsoft.com}
\icmlcorrespondingauthor{Yuxiong He}{yuxhe@microsoft.com}

\icmlkeywords{Machine Learning, ICML}

\vskip 0.3in
]



\printAffiliationsAndNotice{}  % leave blank if no need to mention equal contribution

\begin{abstract}
In recent years, large pre-trained Transformer-based language models have led to dramatic improvements in many natural language understanding tasks. To train these models with increasing sizes, many neural network practitioners attempt to increase the batch sizes in order to leverage multiple GPUs to improve training speed. However, increasing the batch size often makes the optimization more difficult, leading to slow convergence or poor generalization that can require orders of magnitude more training time to achieve the same model quality. In this paper, we explore the steepness of the loss landscape of large-batch optimization for adapting pre-trained Transformer-based language models to domain-specific tasks and find that it tends to be highly complex and irregular, posing challenges to generalization on downstream tasks. 

To tackle this challenge, we propose \name, a novel and efficient method to accelerate the adaptation speed of pre-trained transformer networks. Different from prior methods, we take a sequential game-theoretic approach by adding lightweight adversarial noise into large-batch optimization, which significantly improves adaptation speed while preserving model generalization. 
Experiment results show that \name attains 2.7--9.8$\times$ adaptation speedups over the baseline for GLUE on BERT$_{base}$ and RoBERTa$_{large}$, while achieving comparable and sometimes higher accuracy than the state-of-the-art large-batch optimization methods. Finally, we also address the theoretical aspect of large-batch optimization with adversarial noise and provide a theoretical convergence rate analysis for \name using techniques for analyzing non-convex saddle-point problems. 
\end{abstract}

% \extrafootertext{Author contributions listed at \hyperref[sec:contribution]{the end of paper.}}

% \vspace{-2mm}
\section{Introduction}
\label{sec:intro}
\vspace{-2mm}

There has been a large paradigm shift in AI: large-scale foundation models~\cite{foundation-models}, such as BERT~\cite{bert} and GPT-3~\cite{gpt-3}, are trained on massive open-domain data at scale and then are adapted to a wide range of domains with additional task-specific data. 
Such a paradigm has led to accuracy breakthroughs in many challenging Natural Language Processing (NLP) tasks such as the General Language Understanding Evaluation (GLUE) benchmark~\citep{glue}. Despite their remarkable performance in accuracy, given that these models often contain a few hundred million parameters (e.g., BERT) to over billions of parameters (e.g., GPT-3), enormous challenges have been raised in terms of their training efficiency. 

To accelerate the training speed of large models,
a common practice is to increase the batch size in the optimization algorithm in order to leverage multi-device training~\citep{ddp,roberta,gpipe,mesh-tensorflow,megatron-lm,zero-optimizer}. A larger batch size supports a higher data parallelism degree that allows more workers (e.g., GPUs) to participate in computing gradients locally and then aggregate. 
Furthermore, most operators used in transformer networks are highly optimized in modern linear algebra frameworks. They can scale to larger batch sizes without significantly increasing the time per step~\citep{benchmark-gpu,scaleing-law-nlp}. 
% If practitioners can train each neural network with more hardware and increased speed, then it makes it possible to achieve better results by training even larger models, using larger datasets and exploring new ideas more rapidly. 

However, changing the batch size is not always straightforward, as it often makes optimization more difficult. You et al. observe that increasing batch sizes during the pre-training stage of transformer networks can easily lead to slow convergence and propose LAMB~\citep{lamb} optimizer to speed up large-batch optimization. LAMB applies layer-wise normalization before applying gradient updates, which has been used to successfully train BERT on 1024 TPU chips in 76 minutes.
Despite showing promising results, prior work primarily focuses on accelerating pre-training. However, the adaption stage still incurs non-trivial overhead (e.g., it takes tens of hours to fine-tune RoBERTa-large on MNLI~\citep{glue}) and becomes more expensive as model size increases. One natural idea is to increase the batch size during adaptation. This motivation seems well-aligned with existing works on large-batch optimization. However, our analysis indicates that the loss landscape of large-batch optimization for adapting pre-trained Transformer-based language models to domain-specific tasks is highly complex and irregular, posing challenges to generalization on downstream tasks. 

To address this challenge, we make the following contributions: (1) We present a novel algorithm \name 
that injects lightweight adversarial noise into large batch optimization to speed up the adaptation of pre-trained transformer networks. To the best of our knowledge, this is the first effort to accelerate the adaptation with large-batch adversarial noise.
(2) We conduct extensive evaluation, and our results show that \name accelerates the adaptation of pre-trained Transformer-networks by 2.7--9.8 times over the baseline on BERT~\citep{bert} and RoBERTa~\citep{roberta} over a wide range of natural language understanding (NLU) tasks. We also perform detailed ablation studies to assess the impact of our approach on both generalizability and speed. 
(3) We present a theoretical convergence rate analysis using techniques for analyzing non-convex saddle-point problems.

\vspace{-2mm}
\section{Background and Related Work}
\label{sec:background}
\vspace{-1mm}

Despite the great success of pre-trained transformer networks such as BERT~\citep{bert}, a big challenge, in general, comes from the training efficiency -- even with self-attention and parallelizable recurrence~\citep{transformer}, and high-performance hardware~\citep{tpu}, training transformer networks can still take a significant amount of time. One effective approach to reducing training time is through data parallelism~\citep{bert,roberta,megatron-lm}, which motivates studies on large-batch stochastic non-convex optimizations for transformer networks~\citep{lamb}. These studies have raised concerns with respect to its convergence, generalizability, and training stability by observing that training with a large batch could be difficult~\citep{on-large-batch-training,train-longer,large-batch-reality-check}. Different from prior works, which mostly focus on reducing the pre-training time~\citep{lamb,pld,progressive-stacking,electra}, this work shows an effective approach to accelerate the adaptation of pre-trained models while preserving the accuracy of downstream tasks.

Efficient adaptation of pre-trained Transformer models has also been studied by several recent works including~\citep{adaptor,k-adapter,lora,unified-parameter-efficient}. For example, \citep{adaptor} inserts small modules called adapters to each layer of the pre-trained model, and only the adapters are trained during adaptation. \citep{lora} adds low-rank matrices to approximate parameter updates. \citep{adaptor-fusion} shows that it is possible to quickly adapt to new tasks without catastrophic forgetting. These methods have reported achieving comparable performance to standard fine-tuning on different sets of tasks. However, their focus is on reducing the trainable parameters per task such that each task does not need to keep a separate copy of fine-tuned model parameters. Unlike these methods, which still incur full forward/backward computation cost during adaptation, we investigate how to accelerate the adaptation speed through large-batch optimization and adversarial noise.

On a separate line of research, adversarial training was first proposed in the computer vision literature to improve a model's robustness against adversarial attacks~\citep{fgsm,pgd-k}. Recently, there has been some work that shows that adversarial training helps improve model generalizability~\citep{nmt-double-adv-inputs,improve-nlm-via-adv,smart,alum,hessian-based-large-batch,freelb,concurrent-adv,sam,vit-outperform-resnet}. However, very few works examine how adversarial learning helps improve the adaptation speed of pre-trained Transformer models. The work most similar to ours is~\citet{freelb}, which studies reducing the cost of adversarial training by accumulating the gradient of the parameters from each of the ascent steps and updating the parameters only once after $K$ inner ascent steps with the accumulated gradients. Unlike \citet{freelb}, we investigate the interplay between large-batch optimization and adversarial noise for speeding up the adaptation of pre-trained Transformer models.

\vspace{-2mm}
\section{Challenges and Motivation}
\label{sec:analysis}
\vspace{-1mm}

In this section, we present several studies that reveal the key challenges involved in accelerating the adaptation of pre-trained Transformer networks using pre-trained BERT$_{base}$ model on GLUE as an example. These insights further motivate the design of the large-batch optimization approach detailed in Section~\ref{sec:design}. The detailed hardware/software setup is described in Section~\ref{sec:eval}. 

\textbf{Scalability analysis.}
First, we carry out a scalability test by varying the number of GPU workers from 1 to 32, with and without communication. Different from pre-training, the adaptation stage often employs a much smaller batch size (e.g., 16, 32) than pre-training (e.g., 4096)~\citep{bert,roberta}.We choose a batch size 32, as suggested by most literature for BERT fine-tuning~\cite{bert,roberta}, and we divide the samples in the mini-batch among $P$=$\{$1,2,4,8,16,32$\}$ GPUs. If the per-worker batch size (e.g., 16) is larger than the maximum admissible per-worker batch size (e.g., 8), we use local gradient accumulation~\cite{train-imagenet-in-1hour} to avoid running out of memory. Figure~\ref{fig:scalability_mnli} shows the scalability results.
For batch size 32, the training time decreases when P increases from 1 to 4. However, it quickly plateaus and even decreases with more GPUs. We find that this is because when the batch size is small, the communication overhead dominates the total execution time (e.g.,{B=32} vs. {B=32 (no comm)}). The communication overhead is huge, especially when there is cross-machine communication (e.g., from 16 to 32),  hindering the scalability of multi-GPU training. In contrast, by increasing the batch size (e.g., to 1K), the training time keeps decreasing as the number of GPUs increases because an increased batch size reduces the number of all-reduce communications to process the same amount of data and also increases the compute resource utilization per GPU (i.e., increased computation-vs-communication ratio).

\begin{figure*}[!ht]
  \begin{minipage}[c]{0.25\textwidth}
  \centering
    \subfigure[Scalability]{\includegraphics[scale=0.30, keepaspectratio=true]{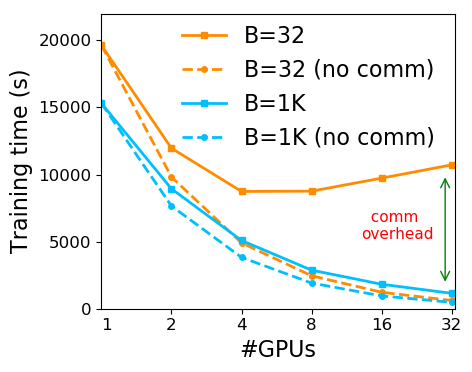}\label{fig:scalability_mnli}}
 \end{minipage}%
%  \hfill
  \begin{minipage}[c]{0.570\textwidth}
  \centering
    \subfigure[Generalizability]{\includegraphics[scale=0.34, keepaspectratio=true]{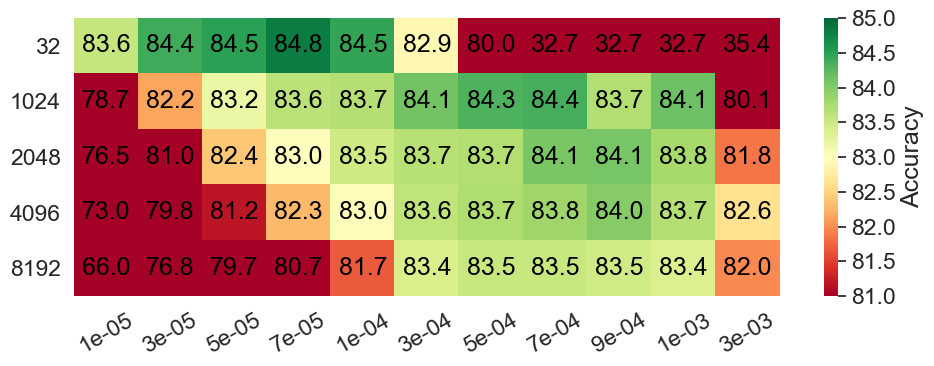}\label{fig:generalizability_mnli}}
 \end{minipage}%
%   \hfill
  \begin{minipage}[c]{0.15\textwidth}
  \centering
    \subfigure[Sharpness]{\includegraphics[scale=0.3, keepaspectratio=true]{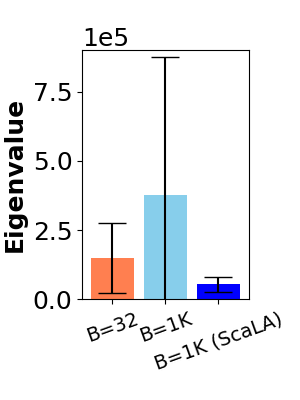}\label{fig:eigenvalue}}
 \end{minipage}%
 \caption{Scalability, generalizability, and curvature analysis results by adapting BERT$_{base}$ to the MNLI task.}
 \label{fig:heatmap-varying-batch}
\end{figure*}

\textbf{Generalizability analysis.}
Increasing the batch size leads to accelerated per-epoch execution time due to the efficient utilization of hardware. However, how would increasing the batch size affect the generalizability in adapting transformer networks? 
Since prior works on batch size scaling often focus on computer vision tasks and pre-training~\cite{do-not-decay-lr,train-imagenet-in-1hour,disciplined-hyperparameters,lamb}, we conduct an analysis of large-batch adaptation on pre-trained Transformers by performing a hyperparameter sweep on batch sizes $\{$1K, 2K, 4K, 8K$\}$ and learning rates $\{$1e-4, 3e-4,5e-4, 7e-4, 9e-4, 1e-3, 3e-3$\}$, where the learning rate range covers both linear scaling~\citep{train-imagenet-in-1hour} and sqrt scaling~\citep{lamb}. 
We report the validation accuracy in Figure~\ref{fig:generalizability_mnli}.
We make two observations: (1) the learning rate scales roughly with the square root of the increase of the mini-batch size, although the best learning rates do not always follow the sqrt rule; (2) there is a generalization gap between the small batch and large batch accuracies, and the gap becomes larger when the batch size increases. Furthermore, methods, such as LAMB~\cite{lamb}, works well on pre-training with extremely large batch sizes ($\log_2 B = \{15, 16\}$) but do not close the generalization gap in adaptation (as shown in Section~\ref{sec:eval}). These results pose the question: can we increase the batch size during adaptation in the interest of making adaptation more efficient but preserving generalization?

\textbf{Curvature analysis.} To further examine the generalization gap, we resort to the curvature analysis.
Prior work~\cite{on-large-batch-training,limit-of-large-batch} correlate the low generalization with sharp minima (which are characterized by a positive curvature of large magnitude in the parameter space). The indication is that a sharp local minimum also reflects a higher sensitivity of the loss even within the neighborhood of training data points and can attribute to the difficulty in generalization. 
Their hypothesis was that a larger noise due to the higher variance in gradient estimates computed using small mini-batches, in contrast to gradient estimates computed using large mini-batches, encourages the parameter weights to exit out of the basin of sharp minima and towards flatter minima which have better generalization.

To verify this hypothesis, we quantitatively measure the steepness of loss landscape by loading the checkpoint of an adapted model and computing the curvature, i.e., properties of the second derivative of the model, with respect to its parameters, for a fixed batch of samples. 
Following \cite{hessian-based-large-batch}, for a model $\Phi(x)$, we compute the largest eigenvalue of the model's Hessian, $L_{\max} [\nabla_x^2 \Phi (x)]$, 
using the Hessian-vector product primitive and the power method. We use the largest eigenvalue as a measure of sharpness since the corresponding (top) eigenvector characterizes the direction of the largest change in gradient at a given point in the parameter space. From Figure~\ref{fig:eigenvalue}, the largest eigenvalue of the model trained with a large batch (e.g., 1K) is much larger (e.g., 2.6x) than the small-batch baseline and with higher deviations (e.g., 3.9x). This result confirms that large-batch adaptation makes the loss landscape of the model more prone to ill-conditioning and less robust to perturbation, which helps explain the loss in generalization.

\vspace{-2mm}
\section{The Proposed Method}
\label{sec:design}
\vspace{-1mm}

Motivated by the challenges in accelerating the adaptation, in this section, we present a principled large-batch optimization method via lightweight adversarial noise for improved adaptation speed while maintaining the quality of the solutions as measured by task-appropriate accuracy metrics. 

\vspace{-2mm}
\subsection{A Sequential Game-theoretic Method via Adversarial Noise}
\label{sec:formulation}
\vspace{-1mm}

% \textbf{Formulation: }
Let $\mathbb{X}$ denote the parameter space and $\mathbb{Y}$ denote the data (mini-batch/sample) space and $Q$ denote a distribution supported on $\mathbb{Y}$. To improve the adaptation speed of pre-trained transformer models while retaining generalizability, we augment the usual stochastic optimization objective by constructing an adversarial~\citep{on-large-batch-training,pgd-k} regularization. In particular, we solve the following optimization problem, which is a stochastic minimax~\citep{lin2020gradient} optimization problem augmented with a regularization term involving a deterministic adversarial noise, instead of vanilla risk minimization: 
%that is deterministic in the data space
%= {\mathbb{E}}_{y \sim \mathbb{Y}}[f(x,y)] + \lambda \max_{\|\delta\|\le \epsilon}\phi(\theta,\delta;x)
\begin{comment}
\begin{align}
    \min_{x \in \mathbb{X}} \mathbb{E}_{\xi \sim Q}[g(x, \xi)] & = \min_{x \in \mathbb{X}} \mathbb{E}_{\xi \sim Q}[\underline{f}(x, \xi) + \lambda \underline{r}(x)] \nonumber \\
    & = \min_{x \in \mathbb{X}} \mathbb{E}_{\xi \sim P} \left[ \underline{f}(x, \xi) + \lambda \max_{ \{ y \in \mathbb{Y} | \| y-\xi \|_\infty \leq \delta \} } r(x, y) \right] \nonumber \\
    & = \min_{x \in \mathbb{X}} \max_{ \{ y \in \mathbb{Y} | \| y-\xi \|_\infty \leq \delta \} } \mathbb{E}_{\xi \sim Q}[\underline{f}(x, \xi) + \lambda r(x, y)] \nonumber \\
    & = \min_{x \in \mathbb{X}} \max_{ \{ y \in \mathbb{Y} | \| y-\xi \|_\infty \leq \delta \} } \mathbb{E}_{\xi \sim Q}[f(x, y, \xi)] \label{eqn:formulation}
\end{align}
\end{comment}
%& = \min_{x \in \mathbb{X}} \mathbb{E}_{\xi \sim Q}[\underline{f}(x, \xi)] + \lambda \underline{r}(x) \nonumber \\
\begin{align}
    \min_{x \in \mathbb{X}} \mathbb{E}_{\xi \sim Q}[g(x, \xi)] 
    & = \min_{x \in \mathbb{X}} \mathbb{E}_{\xi \sim Q}[\underline{f}(x, \xi) + \lambda \underline{r}(x)] \nonumber \\
    %= \min_{x \in \mathbb{X}} \mathbb{E}_{\xi \sim Q} \left[ \underline{f}(x, \xi) + \lambda \max_{ y \in \mathbb{Y} } r(x, y) \right] \nonumber \\
    & = \min_{x \in \mathbb{X}} \max_{ y \in \mathbb{Y} } \mathbb{E}_{\xi \sim Q}[\underline{f}(x, \xi) + \lambda r(x, y)] \nonumber \\
    & := \min_{x \in \mathbb{X}} \max_{ y \in \mathbb{Y} } \mathbb{E}_{\xi \sim Q}[f(x, y, \xi)] \label{eqn:formulation}
\end{align}

\begin{comment}
\begin{equation}
\label{eqn:formulation}
    \min_{\theta}\frac{1}{P}\sum_{i=1}^P\{{\mathbb{E}}_{x,y \in D_i}[\ell(\theta;x,y)] + \lambda \max_{\|\delta\|\le \epsilon}\phi(\theta,\delta;x) \}
\end{equation}
\end{comment}

where $g: \mathbb{X} \times \mathbb{Y} \to \mathbb{R}$ denotes the overall training objective, $\underline{f}: \mathbb{X} \times \mathbb{Y} \to \mathbb{R}$ denotes the standard training objective, $f:\mathbb{X} \times \mathbb{Y} \times \mathbb{Y} \to \mathbb{R}$ denotes the augmented objective, $\underline{r}: \mathbb{X} \to \mathbb{R}$ denotes a deterministic regularization term on the parameters controlled by a strength factor $\lambda \in (0, \infty)$, $r: \mathbb{X} \to \mathbb{R}$ denotes the augmented regularization and $\xi$ denotes samples drawn from $Q$ (for simplicity, we slightly abuse the notation in using $\xi$ to denote the random variable, e.g. $\mathbb{E}_\xi[g(x,\xi)]$, or its empirical realizations, e.g. $\frac{1}{K} \sum_{k=1}^{K} g(x,\xi_k)$ for any $K$).
The overall (outer) training objective involves a minimization problem in the parameter space while being stochastic with respect to the data space. The adversarial regularization (inner) term is a deterministic maximization problem operating in the data space conditioned on a fixed parameter configuration. We emphasize that this formulation is a two-player sequential~\citep{jin2020local}, not simultaneous, game wherein the goal is to optimize a transformer network that is insensitive to adversarial noise. In a given round, the first player (associated with the outer minimization) proposes a parameter configuration, and the second player (associated with the inner maximization) responds with a penalty to capture the effect of label errors due to noises in a large data batch size to undermine the performance of the transformer parameter configuration chosen by the first player.

Language expressions are quite sensitive to individual words or clauses, where noises against those would likely generate incorrect or biased training data with wrong labels~\citep{word-embedding-perturbation}.
Following prior success in applying adversarial training to NLP models~\citep{adv-training-for-semi-supervised,freelb}, we apply noises to the continuous word embeddings instead of directly to discrete words or tokens. The term $\underline{r}$ captures the prediction deviation from the noise. In a given round of the game, with respect to the first player's proposal, let
$\Phi$ 
denote the transformer network under consideration and $\xi$ be a large batch of data sampled from $Q$. We construct a label for the second player as $\gamma:=\Phi(x, \xi)$. Next, for classification tasks, we choose $r$ to be the symmetric KL divergence~\citep{smart}, i.e., $r(x,y):= \text{KL$_{\text{sym}}$}(\gamma, \Phi(x, y))$. We use symmetric KL divergence to measure the distributional divergence to generate adversarial noise. 
For regression tasks, we choose $r$ to be the squared loss, i.e., $r(x,y):= (\gamma - \Phi(x, y))^2$. In practice, we add an $\ell_\infty$ constraint on $y$, which is achieved by simple clipping with a radius of $\omega$ (projection). Intuitively, a large $r$ corresponds to a situation wherein the transformer is highly sensitive to a given noise in the input, suggesting that the model parameters are close to a sharp minimum. Augmenting the original training objective with $r$ makes the first player incur an additional penalty if the outer minimization solution veers closer to sharp minima, thereby encouraging flatter solutions and better generalizability.

\begin{figure*}[t]
%  \centering
  \begin{minipage}[c]{0.36\textwidth}
 %\centering
    \includegraphics[scale=0.7,keepaspectratio=true]{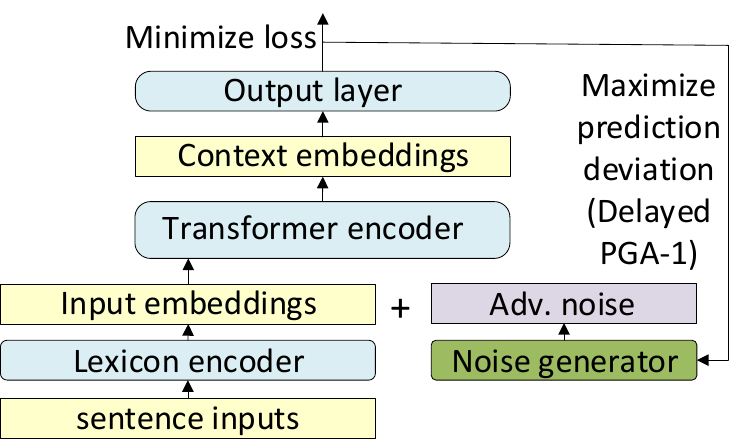}
    \caption{The architecture of the proposed method.}
    \label{tbl:training-time-comparison}
  \end{minipage}
  \hfill
    \begin{minipage}[c]{0.3\textwidth}
 %\centering
    {\includegraphics[scale=0.33, keepaspectratio=true]{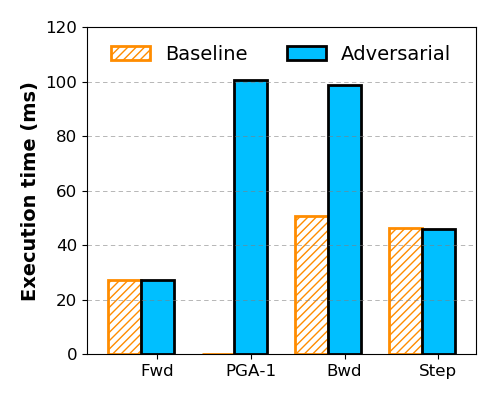}
    %\label{fig:scalability-varying-gpus}}
    \caption{Time breakdown without and with PGA-1.}
    \label{fig:adv_training_cost_breakdown}}
  \end{minipage}
    \hfill
  \begin{minipage}[c]{0.31\textwidth}
 %\centering
    {\includegraphics[scale=0.35, keepaspectratio=true]{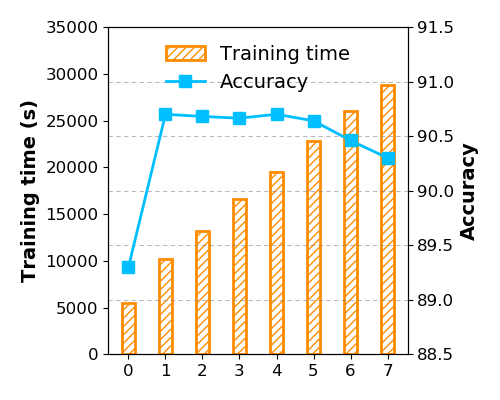}
    \caption{Impact of perturbation steps.}
    %\label{fig:MNLI-m-heatmap}
    \label{fig:perturbation-steps}}
  \end{minipage}
%   \begin{minipage}[c]{0.6\textwidth}
%  %\centering
%   \subfloat[Scalability]{\includegraphics[scale=0.33, keepaspectratio=true]{figs/adv_training_cost_breakdown}
%     %\label{fig:scalability-varying-gpus}}
%     \label{fig:adv_training_cost_breakdown}}
%     \subfloat[Accuracy]{\includegraphics[scale=0.33, keepaspectratio=true]{figs/varying-perturbation-steps.png}
%     %\label{fig:MNLI-m-heatmap}
%     \label{fig:perturbation-steps}}
%   \caption{The fine-tuning results with different batch sizes.}\label{fig:finetune-analysis}
%  \end{minipage}%
\end{figure*}

% \textbf{Inner maximization. }%To solve Equation~\ref{eqn:formulation}, we need to first solve the inner maximization problem. However, it cannot be solved optimally because exact maximization with respect to $\delta$ is intractable for non-convex models such as neural networks.
For any given outer step $t$, let $x_t$ denote the parameter proposed by the first player. Since the exact inner maximization in Equation~\eqref{eqn:formulation} is intractable for non-convex models, we adopt truncated methods as in prior works. Specifically, we use Projected Gradient Ascent (PGA)~\citep{pgd-k,smart} to solve this problem, i.e., $y_{\tau+1}=\Pi_\omega(y_\tau + \rho_\tau \nabla_{y}r(x_t,y))$
\begin{comment}
\begin{equation}
    \delta^{k+1}=\prod(\delta^k + \alpha \cdot \nabla_{\delta^k}{\phi(\theta, \delta^k; x)})
\label{eqn:maximization}
\end{equation}
%$\mathcal{T}$
\begin{equation}
    y_{\tau+1}=\Pi(y_\tau + \rho_\tau \nabla_{y}r(x_t,y))
\label{eqn:maximization}
\end{equation}
\end{comment}
where $\rho_\tau$ for $\tau \in [\mathcal{T}]$ is the step size sequence and $\Pi$ projects the result of the gradient ascent update into an $\ell_\infty$ ball of diameter $2 \omega$ around the original input embeddings, $\xi$, considered by the first player.

\textbf{Practical considerations. }
Inspired by LAMB~\citep{lamb}, which employs group-wise adaptive learning rates to improve the adaptivity and convergence of large-batch optimizations during pre-training, we solve the outer minimization problem in Equation~\eqref{eqn:formulation} with
% \begin{equation}
%     \theta_{t+1,h} = \theta_{t,h} - \frac{\tau(|\ \theta_{t,h}\|_2) \cdot \eta_t}{\|u_{t,h}\|_2} \cdot u_{t,h}, \forall_h \in [H] 
% \end{equation}
%\begin{equation}
%\colorbox{pink}{
$x^i_{t+1} = x^i_{t} - \eta_t \nu(\| x^i_{t}\|) \widehat{\nabla}^i_x g(x) / \|\widehat{\nabla}^i_{x} g(x)\|$,  $\forall i \in [h]$
%\end{equation}
%and $h$ is the height of the network
%$x^i_{t}$
where $i$ denotes the $i^{\text{th}}$-layer of the pre-trained transformer.
% The normalized gradient descent mitigates issues due to exploding gradients. % $\frac{g_{t,h}}{\|g_{t,h}\|}$
%as well as preventing gradient descent from halting near saddle points
The learning rate sequence $\eta_t$, $\forall t \in [T]$ is scaled by a clipping function $\nu(c) := \max(\mathcal{L}, \min(c, \mathcal{U}))$ where $\mathcal{L} < \mathcal{U}$ (e.g., $\mathcal{L}=0$ and $\mathcal{U}=10$), which ensures the norm of the update is of the same order as that of the weights. We show that \name is beneficial to accelerate the adaptation under the large-batch regime with and without group-wise adaptive learning rates, but they can be combined together to deliver better results.  
%The normalization and scaling are done groupwise (every layer forms a group), hence the adaptive layer-wise learning rates.
% We remark that, LARS has not been applied to transformer fine-tuning, and we are the first to apply adaptive optimization techniques to large-batch fine-tuning with adversarial regularization. 
% Note that we use gradient averaging on $\xi$, i.e., gradient accumulation and all-reduce, over a batch size $B$ distributed across $P$ workers in order to obtain a noisy gradient estimate $\widehat{\nabla}^i_x g(x)$ at epoch $t$. %, since a large batch size does not fit in the memory.
% $u_t$ is a descent direction computed by the exponential moving average of past gradients scaled by square root of exponential moving averages of squared past gradients,
% % based on the first-order information $Q(\hat{g}_t)$,
% $\tau(\|\theta_{t,h}\|_2)$=$min\{max\{\|\theta_{t,h}\|_2, c_l\}, c_u\}$ is a layerwise scaling factor of the adaptive learning rate $\frac{\eta_t}{\|u_{t,h}\|_2}$, $c_l=0$ and $c_u=10$, and $\theta_t = [\theta_{t,1}, ...,\theta_{t,h}]$.

% We would like remark that (1) No one has evaluated LAMB against fine-tuning tasks. (2) We combine LAMB with adversarial training for large-batch min-max optimization.

% \minjia{TODO: Need revision}

% \paragraph{Learning rate schedule} We empiricially found that the square root learning rate with the batch size still hold for large-batch adversarial training, at least across the tested batch sizes.

\subsection{Reducing Adversarial Noise Overhead}
\label{subsec:reduce-adv-overhead}

Our experiments find that although injecting adversarial noise into large-batch optimization helps improve the generalizability; it may not reduce the adaptation time because the generation of adversarial noises can take a large fraction of time. This section first provides an analysis of the computational cost and then describes two approaches to reduce the time spent in generating adversarial noise, thereby further reducing the overall adaptation time.

The generation of adversarial noise requires an extra PGA inner loop that standard training does not have. Figure~\ref{fig:adv_training_cost_breakdown} provides the time breakdown of optimization using PGA with $\mathcal{T}=1$ (denoted as PGA-1). PGA-1 performs the perturbation and takes approximately the same time as making three forward passes (Fwd) through the network. This is because one step of PGA requires to make one forward and backward pass (Bwd) over the entire network. The backward pass of the optimization takes roughly twice the amount of time as the standard backward step because the back-propagation is triggered twice to calculate the noise and the gradients. The time spent on the optimizer step function remains the same.
In total, the optimization would slow down training by at least 2 times, even with $\mathcal{T}$=1. 
% Worse, PGA is performed for each mini-batch, where the expensive cost is added to every single step of the training process. 
This motivates us to look at the effectiveness of different perturbation steps as well as the usefulness of perturbation from the initial epochs.

\begin{comment}
\begin{figure}[h!]
 \centering
   \includegraphics[scale=0.5, keepaspectratio=true]{figs/adv_training_cost_breakdown}
  \caption{Time breakdown of adversarial training.}
  \label{fig:adv_training_cost_breakdown}
\end{figure}

\begin{figure}[h!]
 \centering
    \includegraphics[scale=0.5, keepaspectratio=true]{figs/varying-perturbation-steps.png}
    \caption{Impact of perturbation steps to accuracy and computation efficiency. }
    \label{fig:perturbation-steps}
\end{figure}
\end{comment}

\textbf{Impact of perturbation steps. }
%(namely PGA-$\mathcal{T}$
Prior works often do multiple gradient computation steps ($\mathcal{T}>1$) and take several times longer training time to produce adversaries~\citep{pgd-k,freelb}, likely because their focus is on generalization instead of computational efficiency. Subsequently, researchers presented Curriculum Adversarial Training (CAT)~\citep{curriculum-at} and Annealing-based Adversarial Training~\citep{amata}, which progressively increase the perturbation with various strengths, cutting the adversarial training cost while maintaining good accuracy. To investigate how CAT and similar methods affect 
large-scale NLP problems involving transformers, we 
evaluate the final accuracy and training cost of QNLI, varying the number of perturbation steps $\mathcal{T}$ and report the results in Figure~\ref{fig:perturbation-steps}. Interestingly, although using a large $\mathcal{T}$ helps to produce stronger noises, we find that this does not lead to improved accuracy, despite the fact that the training overhead still increases almost linearly. In fact, the best accuracy is achieved with $\mathcal{T}=1$. 

% \minjia{admit the improving the affordability work in computer vision}

We note that the model has two components, namely, the parameter space and data space. First, unlike the minimization in the parameter space, which is stochastic, the maximization in the data space is deterministic. Second, with respect to the testing phase, the numerical convergence in the model's parameter space is of primary importance rather than the numerical convergence in the data space, i.e., the maximization is an auxiliary procedure that augments the training phase to make the parameter space "aware" of effects of the batch size across epochs. Due to these two points, at a certain epoch, for a given batch, the marginal utility of an additional PGA step is low, and we are able to get away with inexact deterministic maximization. Therefore, we apply PGA-1 in our large-batch optimization scheme, given that it produces sufficiently good solutions while being much more computationally efficient. 

\textbf{Delayed noise injection.} Given that PGA-1 still adds an overhead factor of 2, we are motivated to further reduce the overhead of adversarial noise. In particular, we investigate how useful adversarial noises are in the whole large-batch optimization process. 
We conduct additional experiments to measure the final accuracy corresponding to starting
from a regular fine-tuning and then enabling PGA-1 for $t \geq t_s$ where $t_s \in [T]$. Our observation is that enabling PGA-1 from the beginning does not offer much improvement in accuracy, whereas adversarial noise becomes more potent as the model begins to stabilize towards the end of training (more detailed results in Appendix~\ref{subsec:perturbation-usefulness}). In general, at initialization, the model's parameters are relatively far from their final values and are less likely to get stuck at local minima. Therefore the adversarial noises generated in the initial training iterations are quite different from the noises towards the end of training because they would not maximize the adversarial loss in Equation~\ref{eqn:formulation}. This hypothesis suggests that we might be able to inject adversarial noise in the later training process while still leveraging it to improve generalizability. We remark that this phenomenon has been observed by prior work on computer vision tasks~\citep{curriculum-at,delayed-adv}.

%$\omega$
%\subsection{Putting It Together}
\textbf{Putting it together.} Combining the formulation with the above investigations, the full procedure of \name is provided in Algorithm~\ref{algo:algorithm}, whose convergence rate is characterized in Theorem~\ref{thm:oracle}.%In previous sections, we describe how we make fine-tuning process of pre-trained Transformer networks more efficient and with improved generalizability. In this part, we put everything together and call the resulting method \name (Algorithm~\ref{algo:algorithm}).

\begin{algorithm}[ht]
\caption{\hfill \textbf{\name}}
\label{algo:algorithm}
  \begin{algorithmic}[1]
    \STATE \textbf{Input:} Epochs $T$, delay $t_s$, perturbation (inner) step size $\rho$, clipping radius $\omega$, regularization strength $\lambda$, (outer) learning rate $\eta$
    \STATE \textbf{Output}: $h$-layer transformer model $\Phi$ with converged robust parameters $\overline{x}:=x_T$
    % \STATE $x \leftarrow x_0$ (here, $x_0$ denotes the pre-trained model weights)
    \FOR{$t \in [T]$} 
        \FOR{worker $p \in [P]$} 
            % \STATE \COMMENT{In parallel across P workers}
            \FOR {mini-batch $\xi_p \sim Q$} 
                % \STATE \COMMENT{Subsample $\frac{B}{P}$ data instances on each worker}
                %\STATE Obtain a slice of minibatch ($x_p,y_p$)
                % 
                %\STATE Calculate the standard loss $\ell(\theta;x_p,y_p)$
                %\STATE $r(x, y) \leftarrow 0$
                \STATE {$\underline{r}(x_{t}) \leftarrow 0$, $\gamma \leftarrow \Phi(x, \xi_p)$, select $y_0$}
                % \COMMENT{Initialize regularization and label}
                % \colorbox{orange}
                \IF {\colorbox{orange}{$t \geq t_s$}}
                    \STATE \COMMENT{Check delay condition}
                    % \colorbox{lightgray}
                    \STATE \colorbox{lightgray}{$y_{1} \leftarrow \Pi_\omega(y_0 + \rho \nabla_{y}r(x_t,y))$} \COMMENT{Generate adversarial noise with PGA-1} %Perturb $x_p$ with $\delta$ by solving the maximization problem in Equation~\ref{eqn:formulation} with PGD-1    
                    % \colorbox{yellow}
                    \STATE {$\underline{r}(x_{t}) \leftarrow \colorbox{yellow}{\text{KL$_{\text{sym}}$}}(\gamma, \Phi(x_{t-1}, y_1))$}
                    % \COMMENT{Calculate the adversarial regularization}
                \ENDIF
                % \colorbox{lime}
                \STATE \colorbox{lime}{$g(x_t, \xi_p) \leftarrow \underline{f}(x_{t-1}, \xi_p) + \lambda \underline{r}(x_{t})$} 
                % \COMMENT{Calculate the augmented loss}
                \STATE $\nabla_x g(x_t, \xi_p) \leftarrow$ Backward pass on $\Phi$ 
                % \COMMENT{Compute local gradients using accumulation}
            \ENDFOR
        \ENDFOR
        \STATE $\widehat{\nabla}_x g(x_t) \leftarrow \frac{1}{B}\sum_{p=1}^P \nabla_x g(x_t, \xi_p)$ 
        % \COMMENT{Gradient averaging using all-reduce}
        % \colorbox{pink}
        \STATE \colorbox{pink}{$x^i_{t} \leftarrow x^i_{t-1} - \eta_t \nu(\| x^i_{t}\|) \frac{\widehat{\nabla}^i_x g(x_t)}{\|\widehat{\nabla}^i_{x} g(x_t)\|}$,  $\forall i \in [h]$} 
        % \COMMENT{Update model parameters}
    \ENDFOR
  \end{algorithmic}
\end{algorithm}

\begin{comment}
We calculate the below approximated adversarial perturbation:

\begin{equation}
    r_{adv} = g/||g||_2\ where \ g = \nabla_{x + r_{adv}} KL[p(\cdot |x; \theta) || p(\cdot|x + r_{adv}; \theta) ]
\end{equation}

where $r_{adv}$ is a small random vector. This approximation corresponds to a 2nd order Taylor expansion and a single iteration of the power method. We define the adversarial loss to be 

\begin{equation}
    L_{adv}(\theta) = \frac{1}{B} \sum_{b=1}^{B} KL[p(\cdot |x_b; \theta) || p(\cdot|x_b + r_{adv, b}; \theta) ]
\end{equation}
where B is the number of mini-batch samples.

However, we cannot calculate this value exactly in general, because exact minimization with respect to $r$ is intractable for models such as neural networks.  The additional cost introduced by virtual adversarial training is the following:
\end{comment}

\begin{theorem}[Complexity of Algorithm~\ref{algo:algorithm}; Informal -- Details in Appendix~\ref{sec:proof}]
\label{thm:oracle}
Consider the problem in Equation~\ref{eqn:formulation}. Let $t_s =0$. Setting the outer learning rate as $\eta = O \left( 1 / \sqrt{T} \right)$ and scaling batch size as $b = O(T)$, for Algorithm~\ref{algo:algorithm}, we have
$
    \mathbb{E} \left[ \| \nabla g_{1/2 \alpha}(\overline{x}) \|^2 \right] \leq O \left( \epsilon  + \kappa_\alpha / \sqrt{T} \right)
$
where $\overline{x}$ is the estimator obtained from running $T$ steps of Algorithm~\ref{algo:algorithm} and picking $x_t$ uniformly at random for $t \in [T]$. Here, $\epsilon$ is the error due to the approximate inner maximization oracle, $\alpha$ characterizes the smoothness of $f(x,.)$, $g_{1/2\alpha}$ is the Moreau-envelope of $g$ and $\kappa_\alpha = \max_i \alpha_i / \min_i \alpha_i$.
\end{theorem}

\vspace{-2mm}
\section{Evaluation}
\label{sec:eval}
\vspace{-2mm}

We evaluate the effectiveness of \name in adapting pre-trained transformer networks over a set of NLP tasks. 
% More comprehensive results can be found in the Appendix.
% We measure the performance of \name in the following aspects: a) accuracy with and without adversarial training, b) scalability to multiple computing nodes.
% \paragraph{DNN models and datasets} We use pre-trained BERT-base from HuggingFace and GLUE (General Language Understanding Evaluation), a collection of 9 sentence or sentence-pair natural language understanding tasks including question answering, sentiment analysis, and textual entailment.
% \paragraph{Computing resources} We study the the computation efficiency and scalability using 2 DGX-2 nodes. Each node consists of 16 Nvidia V100. Nodes are connected together with InfiniBand using a 648-port Mellanox MLNX-OS CS7500 switch.
%\subsection{Experimental Results}
%\label{subsec:experiment-results}
\begin{comment}
\minjia{TODO: Should have compared with SQUDA, so we have direct comparison with LAMB. Without it, the increased accuracy results seem rather surprising.}
\end{comment}
% We first verify the training speed and accuracy of the proposed method.
% \textbf{Implementation details.} 

\textbf{Hardware.} We conduct the evaluation using 2 NVIDIA DGX-2 nodes. Each node consists of 16 NVIDIA V100 GPUs. The nodes are connected with InfiniBand using a 648-port Mellanox MLNX-OS CS7500 switch. 
% We vary the number of workers (i.e., from 1 GPU to 32 GPUs) in the experiments to evaluate the scalability. 
% \textbf{Software:} We use PyTorch Distributed Data Parallel 
% % (DDP~\footnote{\url{https://pytorch.org/tutorials/intermediate/ddp_tutorial.html}}), 
% to scale the fine-tuning from a single GPU to multiple GPUs.
% which uses all-reduce~\citep{proficz2018improving} to compute the average of the gradients across all workers for a parameter update. 
% We use NCCL V2.4 as the underlying all-reduce implementation.
\textbf{Model/Dataset.} We study adaptation on pre-trained \bb model and RoBERTa$_{large}$ hosted by HuggingFace~\citep{hf-transformers}.
% ~\footnote{\url{https://huggingface.co/bert-base-uncased}}. 
We use the GLUE benchmark~\citep{glue}, which is a collection of sentence or sentence-pair natural language understanding tasks including question answering, sentiment analysis, and textual entailment. 
% Since the adaptation time for different tasks varies,
% Given that large models normally require days or even weeks (e.g., RoBERTa) if trained from scratch, it is not cost-efficient to evaluate from scratch using large-scale pretraining. 
% Therefore,
% The experiments presented in this section focus on adaptation pre-trained models, e.g., BERT~\citep{bert} and RoBERTa~\citep{roberta}, and 
We exclude tasks that have very small datasets (e.g.,CoLA, RTE). We report the details about the hyperparameters in Appendix~\ref{sec:hyperparameters}.

\vspace{-2mm}
\subsection{Main Results -- Adaptation Time Acceleration and Generalizability Improvement} 
\vspace{-2mm}

We first compare the following schemes: (1) \textbf{Single GPU + SB: } This is the existing PyTorch implementation of Transformer fine-tuning from HuggingFace (HF), using small batch (SB) sizes (e.g., 32). (2) \textbf{Multi-GPU + SB: } This is multi-GPU PyTorch implementation using DistributedDataParallel~\citep{ddp}, and (3) \textbf{Multi-GPU + LB + \name:} This is our approach as described in Algorithm~\ref{algo:algorithm}, using large minibatches (LB), e.g., 1K, for adaptation. 
Table~\ref{tbl:downstream-accuracy} shows results on MNLI, QNLI, QQP, and SST2, which are larger datasets and less sensitive to random seeds. $n \times g$ refers to $P_n$ nodes each with $P_g$ GPUs for a total of $P = P_n P_g$ homogeneous workers (e.g., 32 GPUs on 2 NVIDIA DGX-2 nodes). For a fair comparison, we reproduce BERT and RoBERTa baseline. Our reproduced baseline achieves the same or slightly higher accuracy than the originally reported results in \citep{bert} and \citep{roberta}. We now discuss our results and observations.

\begin{table*}[!ht]
	\newcommand{\smallcolspc}{\hspace*{0.3em}}
	\newcommand{\colspc}{\hspace*{0.28em}}
    \small
    \renewcommand\sfsmaller{}
    \centering
    \caption{The adaptation time and accuracy results on GLUE benchmark. \name achieves the same average accuracy as the baseline while providing up to 18$\times$ speedups than single GPU, and up to 9.8$\times$ speedups with the same amount of hardware.}
\begin{tabular}{|@{\smallcolspc}c@{\smallcolspc}|@{\smallcolspc}c@{\smallcolspc}|@{\smallcolspc}c@{\smallcolspc}|@{\smallcolspc}c@{\smallcolspc}|@{\smallcolspc}c@{\smallcolspc}|@{\smallcolspc}c@{\smallcolspc}|@{\smallcolspc}c@{\smallcolspc}|@{\smallcolspc}c@{\smallcolspc}|@{\smallcolspc}c@{\smallcolspc}|@{\smallcolspc}c@{\smallcolspc}|@{\smallcolspc}c@{\smallcolspc}|@{\smallcolspc}c@{\smallcolspc}|@{\smallcolspc}c@{\smallcolspc}|@{\smallcolspc}c@{\smallcolspc}|@{\smallcolspc}c@{\smallcolspc}|@{\smallcolspc}c@{\smallcolspc}|}
\hline
\multirow{2}{*}{BERT\textsubscript{base}}         & \multirow{2}{*}{n$\times$g} & \multirow{2}{*}{bsz} & \multicolumn{3}{c|}{\textbf{MNLI-m}} & \multicolumn{3}{c|}{\textbf{QNLI}} & \multicolumn{3}{c|}{\textbf{QQP}} & \multicolumn{3}{c|}{\textbf{SST-2}} & \multirow{2}{*}{Avg.} \\ \cline{4-15} 
                  &                      &                           & Steps     & Time     & Acc.      & Steps     & Time     & Acc.    & Steps   & Time   & Acc/F1     & Steps     & Time     & Acc.   &  \\ \hline
Devlin et al. 2019 &                       &                              &       &                                           & 84.4          &               &                               & 88.4                  &          &                          & \--                              &                &                               & 92.7      &    \--        \\ \hline
Baseline (B=32)                    & 1x1                   & 32                           & 73632 & 19635                                     & 84.8          & 19644         & 5535                          & \textbf{90.6}         & 68226    & 16494                    & \multicolumn{1}{l|}{\textbf{91/88.0}}   & 12630          & 2736                          & 93.1        & \textbf{89.4}          \\ \hline
Baseline (B=32)                               & 2x16                  & 32                           & 73632 & 8848                                      & 84.8          & 19644         & 2408                          & 90.6                  & 68226    & 11311                    & \multicolumn{1}{l|}{91/88.0}   & 12630          & 1494                          & 93.1    & \textbf{89.4}              \\ \hline
% Baseline (B=1K)                               & 2x16                  & 1K                           & 2301  & {1148}                             & 84.3          & 615           & {349}                  & 89.3                  & 2133     & {2892}            & 89.6/86.1
%          & 396           & {134}                  & 93  & 88.4                  \\ \hline
% LAMB (B=1K)                 & 2x16                  & 1K                           & 2301  & 1180 & 84.1          & 615           & 359         & 89.6                  & 2133     & 2978    & 90.5/87.0          & 396           & 139         & 92.4   & 88.6               \\ \hline
% DDP + LB + AT                          & 2x16                  & 1K                           & 2301  & 2503                                      & 85.2          & 615           & 726                           & 90.2                  & 2133     & 6407                     & \textbf{91.3/88.3}                      & 396           & 272                           & 93.1                  \\ \hline
\name (B=1K)                         & 2x16                  & 1K                           & 2301  & \textbf{1323}                             & \textbf{85.1} & 615           & \textbf{432}                  & 90.0                    & 2133     & \textbf{4229}          & 90.9/87.7                      & 396           & \textbf{151}                  & \textbf{93.5}     & \textbf{89.4}    \\ \hline
% \end{tabular}

% \begin{tabular}{|@{\smallcolspc}c@{\smallcolspc}|@{\smallcolspc}c@{\smallcolspc}|@{\smallcolspc}c@{\smallcolspc}|@{\smallcolspc}c@{\smallcolspc}|@{\smallcolspc}c@{\smallcolspc}|@{\smallcolspc}c@{\smallcolspc}|@{\smallcolspc}c@{\smallcolspc}|@{\smallcolspc}c@{\smallcolspc}|@{\smallcolspc}c@{\smallcolspc}|@{\smallcolspc}c@{\smallcolspc}|@{\smallcolspc}c@{\smallcolspc}|@{\smallcolspc}c@{\smallcolspc}|@{\smallcolspc}c@{\smallcolspc}|@{\smallcolspc}c@{\smallcolspc}|@{\smallcolspc}c@{\smallcolspc}|@{\smallcolspc}c@{\smallcolspc}|}
\hline
\multirow{2}{*}{RoBERTa\textsubscript{large}}         & \multirow{2}{*}{n$\times$g} & \multirow{2}{*}{bsz} & \multicolumn{3}{c|}{\textbf{MNLI-m}} & \multicolumn{3}{c|}{\textbf{QNLI}} & \multicolumn{3}{c|}{\textbf{QQP}} & \multicolumn{3}{c|}{\textbf{SST-2}} & \multirow{2}{*}{Avg.} \\ \cline{4-15} 
                  &                      &                             & Steps     & Time     & Acc.      & Steps     & Time     & Acc.    & Steps   & Time   & Acc/F1     & Steps     & Time     & Acc.   &  \\ \hline
Liu et al. 2020 &                       &                              &       &                                           & 90.2          &               &                               & 94.7                 &          &                          & 92.2/\--                              &                &                               & 96.4      &    \--        \\ \hline
Baseline (B=32)                         & 1x1                  & 32                          & 73632   & 43090     & 90.5  & 19644  & 14188     & 94.7 & 68226 & 40945    & 92.0/89.4 & 12630  & 4940      & 96.4  & 92.5    \\ \hline
Baseline (B=32)                           & 2x16                 & 32                          & 73632   & 18114     & 90.5  & 19644  & 4842      & 94.7 & 68226 & 16614    & 92.0/89.4 & 12630  & 3072      & 96.4  & 92.5    \\ \hline
\name (B=1K)                         & 2x16                 & 1K                          & 2301    & \textbf{3363}      & \textbf{90.9}  & 615    & \textbf{1168}      & \textbf{95.1} & 2133  & \textbf{2404}     & \textbf{92.3/89.8} & 396    & \textbf{401}       & \textbf{96.7}  & \textbf{92.9}    \\ \hline
\end{tabular}
\label{tbl:downstream-accuracy}
\end{table*}

%\begin{itemize}
    %\item
    % \textbf{Limited data parallelism hurts training scalability:} 
    \textbf{\name accelerates adaptation time.}
    Compared with single-GPU training, the multi-GPU baseline leads to only modest training speedup improvements, e.g., with $1.5 - 2.4 \times$ faster training speed for both BERT and RoBERTa, even with $32 \times$ more compute resources. The speedup is limited because of the small mini-batches (e.g., 32) used for adaptation, which do not provide a sufficient workload to fully utilize the underlying hardware. Thus, communication overhead becomes the dominant part, and the adaptation often struggles to obtain speedups even with more workers.
    In contrast, \name achieves up to $18 \times$ speedups over the single-GPU baseline with 32 GPUs. When using the same number of GPUs (e.g., 32), \name is 2.7--9.8$\times$ faster. The speedups come from three aspects: (1) the improved hardware efficiency for each worker from increased per-worker micro-batch size; (2) the reduced all-reduce communication overhead since it takes fewer communication rounds to process the same number of samples in one epoch; (3) the lightweight adversarial noise incurs only a small portion of the total training overhead. Finally, \name obtains the speedups while achieving the same accuracy (88.4 vs. 88.4) average accuracy for BERT and higher accuracy (92.9 vs. 92.5) for RoBERTa as the baselines. 
    % Notably, \name is 0.3 points higher (85.1 vs. 84.8) for MNLI and 0.4 points higher (93.5 vs. 93.1) for SST-2 on BERT, and consistently provides 0.3-04 points higher accuracy on RoBERTa. 
    % \name improves generalizability because it introduces adversarial perturbations that drastically change the model's predictions. By training the network to be robust to such perturbations, the model loss landscape is smoothed out, leading to improved generalization.  
    % The speedup is not linear because there is still communication overhead among workers.  
    % 6.7, 5.6, 2.7, 9.8 
    % The increased batch size allows the fine-tuning process to improve the processing throughput in samples per second and achieving even faster speeds with more compute units. 
    
    % \name is around 2.2 times slower than DDP + LB, because the projected gradient descent steps that solve the inner maximization problem adds non-trivial overhead. 
    
    \textbf{\name improves generalizability.} Since there are very few works on large-batch adaptation, we create several baselines to compare with \name: (1) Multi-GPU + LB + Tuning LR: This configuration uses large mini-batches (e.g., 1K), and applies heuristic-based scheduling rule (e.g., square root) combined with an extensive grid search for learning rates; (2) Multi-GPU + LB + LAMB: Uses LAMB~\citep{lamb} optimizer for large-batch adaptation.  
    % (3) Multi-GPU + LB + FreeLb: Applies FreeLb~\citep{freelb}\footnote{The original FreeLb does not support multi-node training. We extend it with PyTorch DDP to train in a distributed training environment.} to fine-tuning tasks. 
    % By increasing the batch size $32 \times$ times to 1024, the training time reduces by up to $21 \times$ in comparison to single-GPU to train the same number of samples.
    We make several observations from the results in Table~\ref{tbl:downstream-accuracy-all-large-batch}. First, compared with the baseline accuracy reported in the paper, the accuracy of Multi-GPU + LB drops by close to 1 point (88.4 vs. 89.4, and 92.1 vs. 92.9) in average and close to 2 points for some tasks (e.g., QQP on BERT), indicating that it is challenging to obtain on-par accuracy with large-batch optimizations for adaptation despite with heavy hyperparameter tuning. Second, 
% LAMB~\citep{lamb} was introduced as a technique to stabilize large-batch optimization and has primarily been evaluated against pre-training transformer networks. 
    since LAMB is designed primarily for improving the convergence of pre-training instead of the adaptation, its ability to accelerate the adaptation has yet to be proven. In our experiments, LAMB leads to only marginal improvements (88.6 vs. 88.4, and 92.1 vs. 92.1) than the baseline and is 0.8 points lower than the small-batch baseline. This is because LAMM does not directly minimize the sharpness of the loss landscape, so it can still lead to poor generalizability during adaptation. With \name, we are able to close the generalization gap from large-batch optimization (89.4 vs. 89.4, and 92.5 vs. 92.9) and achieve 0.8 points higher accuracy (89.4 vs. 88.6, 92.9 vs. 92.1) than LAMB on both BERT and RoBERTa. \name improves generalizability because it introduces adversarial noise in the large-batch optimization process, which serves as a regularizer. By training the network to be robust to such perturbations, the model loss landscape is smoothed out, leading to improved generalization.  
% Notably, compared with LAMB, \name achieves 1 point higher (85.1 vs.84.1) accuracy for MNLI-m, 0.7 points higher (87.7 vs. 87.0) in F1 measure for QQP, and 1.1 points higher (93.5 vs. 92.4) accuracy for SST-2 on BERT, and 1 point higher (92.3 vs. 91.3) accuracy for QQP on RoBERTa. 

% \minjia{TODO: Once we have FreeLb results, summarize the findings here.}

% However, it is less clear how LAMB helps improve generalization of downstream tasks. 

% generalization problems, at least for tested downstream tasks. 
    
    \begin{table*}[!ht]
	\newcommand{\smallcolspc}{\hspace*{0.3em}}
	\newcommand{\colspc}{\hspace*{0.28em}}
    \small
    \renewcommand\sfsmaller{}
    \centering
    \caption{The comparison results between \name and alternative methods for large-batch adaptation on the GLUE benchmark, which show that \name achieves higher accuracy than baselines after training the same number of samples and steps.}
\begin{tabular}{|@{\smallcolspc}c@{\smallcolspc}|@{\smallcolspc}c@{\smallcolspc}|@{\smallcolspc}c@{\smallcolspc}|@{\smallcolspc}c@{\smallcolspc}|@{\smallcolspc}c@{\smallcolspc}|@{\smallcolspc}c@{\smallcolspc}|@{\smallcolspc}c@{\smallcolspc}|@{\smallcolspc}c@{\smallcolspc}|@{\smallcolspc}c@{\smallcolspc}|@{\smallcolspc}c@{\smallcolspc}|@{\smallcolspc}c@{\smallcolspc}|@{\smallcolspc}c@{\smallcolspc}|@{\smallcolspc}c@{\smallcolspc}|@{\smallcolspc}c@{\smallcolspc}|@{\smallcolspc}c@{\smallcolspc}|@{\smallcolspc}c@{\smallcolspc}|}
\hline
\multirow{2}{*}{BERT\textsubscript{base}}         & \multirow{2}{*}{n$\times$g} & \multirow{2}{*}{Batch} & \multicolumn{3}{c|}{\textbf{MNLI-m}} & \multicolumn{3}{c|}{\textbf{QNLI}} & \multicolumn{3}{c|}{\textbf{QQP}} & \multicolumn{3}{c|}{\textbf{SST-2}} & \multirow{2}{*}{Avg.} \\ \cline{4-15} 
                  &                      &      size                       & Steps     & Time     & Acc.      & Steps     & Time     & Acc.    & Steps   & Time   & Acc/F1     & Steps     & Time     & Acc.   &  \\ \hline
Baseline (B=1K)                               & 2x16                  & 1K                           & 2301  & {1148}                             & 84.3          & 615           & {349}                  & 89.3                  & 2133     & {2892}            & 89.6/86.1
         & 396           & {134}                  & 93  & 88.4                  \\ \hline
LAMB (B=1K)                 & 2x16                  & 1K                           & 2301  & 1180 & 84.1          & 615           & 359         & 89.6                  & 2133     & 2978    & 90.5/87.0          & 396           & 139         & 92.4   & 88.6               \\ \hline
% FreeLb (B=1K)                 & 2x16                  & 1K                           & 2301  & 1180 & 85.1          & 615           & 359         & 90.4                  & 2133     & 2978    & 90.5/87.0          & 396           & 139         & 92.4   & 88.6               \\ \hline
\name (B=1K)                         & 2x16                  & 1K                           & 2301  & 1323                             & \textbf{85.1} & 615           & 432                  & \textbf{90.0}                    & 2133     & 4229          & \textbf{90.9/87.7}                      & 396           & 151                  & \textbf{93.5}     & \textbf{89.4}    \\ \hline
% \end{tabular}
% \begin{tabular}{|@{\smallcolspc}c@{\smallcolspc}|@{\smallcolspc}c@{\smallcolspc}|@{\smallcolspc}c@{\smallcolspc}|@{\smallcolspc}c@{\smallcolspc}|@{\smallcolspc}c@{\smallcolspc}|@{\smallcolspc}c@{\smallcolspc}|@{\smallcolspc}c@{\smallcolspc}|@{\smallcolspc}c@{\smallcolspc}|@{\smallcolspc}c@{\smallcolspc}|@{\smallcolspc}c@{\smallcolspc}|@{\smallcolspc}c@{\smallcolspc}|@{\smallcolspc}c@{\smallcolspc}|@{\smallcolspc}c@{\smallcolspc}|@{\smallcolspc}c@{\smallcolspc}|@{\smallcolspc}c@{\smallcolspc}|@{\smallcolspc}c@{\smallcolspc}|}
\hline
\multirow{2}{*}{RoBERTa\textsubscript{large}}         & \multirow{2}{*}{n$\times$g} & \multirow{2}{*}{Batch} & \multicolumn{3}{c|}{\textbf{MNLI-m}} & \multicolumn{3}{c|}{\textbf{QNLI}} & \multicolumn{3}{c|}{\textbf{QQP}} & \multicolumn{3}{c|}{\textbf{SST-2}} & \multirow{2}{*}{Avg.} \\ \cline{4-15} 
                  &                      &      size                       & Steps     & Time     & Acc.      & Steps     & Time     & Acc.    & Steps   & Time   & Acc/F1     & Steps     & Time     & Acc.   &  \\ \hline
Baseline (B=1K)                                  & 2x16                 & 1K                          & 2301    & 2514      & 90.1  & 615    & 936       & 94.3 & 2133  & 1874     & 91.7/89.1 & 396    & 317       & 95.9  & 92.1    \\ \hline
LAMB (B=1K)                          & 2x16                 & 1K                          & 2301    & 2646      & 90.5  & 615    & 973       & 94.5 & 2133  & 1998     & 91.3/88.5 & 396    & 324       & 96.2  & 92.1   \\ \hline
\name (B=1K)                          & 2x16                 & 1K                          & 2301    & 3363      & \textbf{90.9}  & 615    & 1168      & \textbf{95.1} & 2133  & 2404     & \textbf{92.3/89.8} & 396    & 401       & \textbf{96.7}  & \textbf{92.9}    \\ \hline
\end{tabular}
% \minjia{TODO: Add FreeLB comparison results.}
\label{tbl:downstream-accuracy-all-large-batch}
\end{table*}

\vspace{-2mm}
\subsection{Experiment -- Analysis Results}
\vspace{-1mm}
\textbf{Ablation analysis:} In this part, we study the importance of components in \name.
We set $t_s$ to 0, which denotes as \emph{w/o Delaying PGA-1}. We replace the outer minimization to use ADAM~\citep{adam}, which is noted as \emph{w/o Groupwise LR}. We set $\lambda$ to 0, which denotes as \emph{w/o PGA-1}. The results are reported in Table~\ref{tbl:ablation}. 

\begin{figure*}[ht!]
%  \centering
 \begin{minipage}[c]{0.66\textwidth}
    \captionof{table}{Ablation study of \name using BERT$_{base}$ on GLUE tasks.}
\label{tbl:ablation}
	\newcommand{\smallcolspc}{\hspace*{0.28em}}
	\newcommand{\colspc}{\hspace*{0.28em}}
    \small
    \renewcommand\sfsmaller{}
    \centering
\tabcolsep=0.10cm
\begin{tabular}{|@{\smallcolspc}c@{\smallcolspc}|@{\smallcolspc}c@{\smallcolspc}|@{\smallcolspc}c@{\smallcolspc}|@{\smallcolspc}c@{\smallcolspc}|@{\smallcolspc}c@{\smallcolspc}|@{\smallcolspc}c@{\smallcolspc}|@{\smallcolspc}c@{\smallcolspc}|@{\smallcolspc}c@{\smallcolspc}|@{\smallcolspc}c@{\smallcolspc}|@{\smallcolspc}c@{\smallcolspc}|}
\hline
\multirow{2}{*}{}         & \multicolumn{2}{c|}{\textbf{MNLI-m}} & \multicolumn{2}{c|}{\textbf{QNLI}} & \multicolumn{2}{c|}{\textbf{QQP}} & \multicolumn{2}{c|}{\textbf{SST-2}} & \multirow{2}{*}{\textbf{Avg.}} \\ \cline{2-9} 
                          &       Time     & Acc.     & Time     & Acc.   & Time   & Acc/F1    & Time     & Acc.   &  \\ \hline
% BERT~\cite{bert} &                       &                              &       &                                           & 84.4          &               &                               & 88.4                  &          &                          & \--                              &                &                               & 92.7                  \\ \hline
BERT                     & 19635                                     & 84.8         & 5535                          & {90.6}       & 16494                    & \multicolumn{1}{l|}{91/88.0}       & 2736                          & 93.1      & 89.4            \\ \hline
\name                    & 1323                             & {85.1}   & 432                  & 90          & 4229          & 90.9/87.7                       & 151                  & {93.5}  & 89.4       \\ \hline
w/o Delaying PGA-1        & 2503                                      & 85.2          & 726                           & 90.2             & 6407                     & {91.3/88.3}     & 272                           & 93.1       & 89.5           \\ \hline
w/o Groupwise LR          & 1290 & 85.0          & 422         & 89.9     & 4212    & 90.7/87.6      & 146         & 93.0   & 89.2               \\ \hline
w/o PGA-1              & 1180 & 84.1          & 359         & 89.6      & 2978    & 90.5/87.0          & 139         & 92.4   & 88.6  \\ \hline 
\end{tabular}
 \end{minipage}%
 ~
 \begin{minipage}[c]{0.31\textwidth}
    \captionof{table}{Comparison results with FreeLb.}
    \label{tbl:comparison-with-freelb}
	\newcommand{\smallcolspc}{\hspace*{0.28em}}
	\newcommand{\colspc}{\hspace*{0.28em}}
    \small
    \renewcommand\sfsmaller{}
    \centering
\tabcolsep=0.10cm
\begin{tabular}{|@{\smallcolspc}r@{\smallcolspc}|@{\smallcolspc}c@{\smallcolspc}|@{\smallcolspc}c@{\smallcolspc}|@{\smallcolspc}c@{\smallcolspc}|@{\smallcolspc}c@{\smallcolspc}|}
\hline
\multirow{2}{*}{} & \multicolumn{2}{l|}{\textbf{MNLI-m}} & \multicolumn{2}{l|}{\textbf{SST-2}} \\ \cline{2-5} 
                  & Acc.              & Time             & Acc.             & Time             \\ \hline
Baseline          & 84.8              & 8848             & 93.1             & 2736             \\ \hline
FreeLb            & \textbf{85.1}              & 3773             & 93.3             & 389              \\ \hline
ScaLA             & \textbf{85.1}              & \textbf{1323}             & \textbf{93.5}             & \textbf{151}              \\ \hline
\end{tabular}
 \end{minipage}%
\end{figure*}

The results in Table~\ref{tbl:ablation} show that the removal of either design element would result in a performance drop. For example, removing PGA-1 leads to 0.8 points accuracy drop (88.6 vs. 89.4), indicating that adversarial noise is crucial for improving the generalizability of large-batch adaptation. Moreover, if we perform PGA-1 without delayed injection, the average accuracy increases by 0.1 points (89.5 vs. 89.4), but the execution time is increased by 1.5--1.9x, indicating the importance of having lightweight adversarial noise for speeding up the adaptation. Finally, removing group-wise learning rates leads to a small 0.2 points accuracy drop (89.2 vs. 89.4), indicating that \name still achieves benefits without group-wise learning rates (89.2 vs. 88.6), but they are complementary to each other.

\begin{table}[ht!]
	\newcommand{\smallcolspc}{\hspace*{0.28em}}
	\newcommand{\colspc}{\hspace*{0.28em}}
    \small
    \renewcommand\sfsmaller{}
    \centering
    \caption{Alternative methods to generate perturbations using random noise, ground-truth, and label probability.}
    \label{tbl:other-perturbations}
\tabcolsep=0.10cm
\begin{tabular}{|@{\smallcolspc}r@{\smallcolspc}|@{\smallcolspc}c@{\smallcolspc}|@{\smallcolspc}c@{\smallcolspc}|@{\smallcolspc}c@{\smallcolspc}|@{\smallcolspc}c@{\smallcolspc}|@{\smallcolspc}c@{\smallcolspc}|}
\hline
{Model}      & 
\begin{tabular}[c]{@{}l@{}}MNLI-m \end{tabular}  &
\begin{tabular}[c]{@{}l@{}}QNLI \end{tabular} & 
\begin{tabular}[c]{@{}l@{}}QQP \end{tabular} &
\begin{tabular}[c]{@{}l@{}}SST-2 \end{tabular} & 
Avg \\ \hline
Baseline	& 84.3	& 89.3	& 89.6/86.1	& 93  &	88.4  \\ \hline
Gaussian noise    & 84.5          & 89.4          & 90.3/87.0    & 92.6         & 88.7         \\ \hline
ScaLA (GT)      & 84.1          & 89.6          & 90.7/87.6             & 93.2         & 89.0         \\ \hline
ScaLA (LP) & \textbf{85.1}          & \textbf{90}            & \textbf{90.9/87.7}             & \textbf{93.5}         & \textbf{89.4}         \\ \hline
\end{tabular}
\end{table}

\begin{figure*}[!ht]
%  \centering
 \begin{minipage}[c]{0.28\textwidth}
 \centering
   \includegraphics[scale=0.35, keepaspectratio=true]{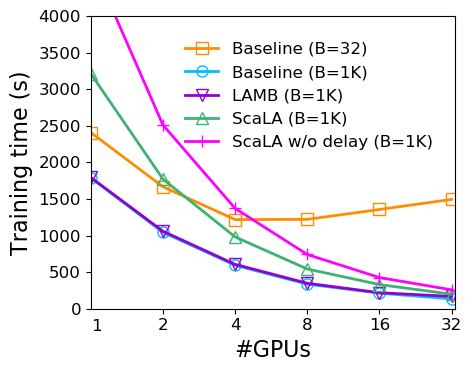}
  \caption{Comparison of scalability using different large-batch optimization methods on SST-2. }
  \label{fig:scalability-varying-gpus-sst}
 \end{minipage}%
 ~
 \begin{minipage}[c]{0.28\textwidth}
 \centering
    \includegraphics[scale=0.33, keepaspectratio=true]{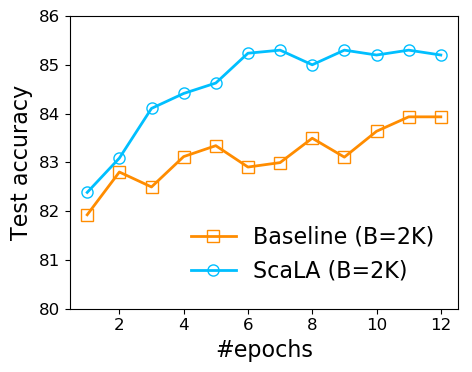}
    \caption{Comparison of test accuracy by training the baseline longer. }
    \label{fig:train-longer-mnli}
  \end{minipage}
  \begin{minipage}[c]{0.4\textwidth}
 %\centering
    \subfigure[MNLI-m]{\includegraphics[scale=0.4, keepaspectratio=true]{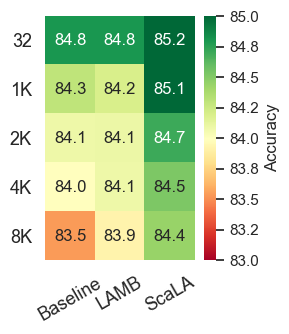}\label{fig:varying-batch-size-MNLI-m}}
    \subfigure[SST-2]{\includegraphics[scale=0.4, keepaspectratio=true]{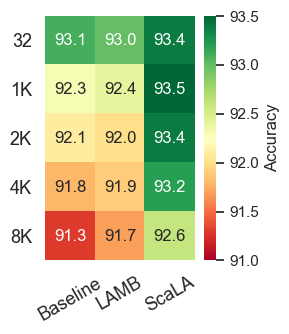}\label{fig:varying-batch-size-SS}}
  \caption{Comparison of accuracy under even larger batch sizes.}\label{fig:varying-batch-comparison}
 \end{minipage}%
\end{figure*}

\textbf{Curvature analysis.} We measure the steepness of the loss landscape again after applying \name. As shown in Fig.~\ref{fig:eigenvalue}, the largest eigenvalue of the model becomes much smaller (6.9$\times$) with lower deviations with \name and is slightly better than the small batch baseline, which is a strong indication that our approach enforces the smoothness of the model that leads to the accuracy improvement. 

\textbf{Comparison with random noise.} We have performed additional experiments by adding Gaussian noise to the embeddings. Table~\ref{tbl:other-perturbations} that random noise indeed can improve the accuracy for MNLI-m (84.3 vs. 84.5), QNLI (89.3 vs. 89.4), and QQP (90.3/87.0 vs. 89.6/86.1) over the baseline, but it also leads to worse results on SST-2 (93. vs. 92.6). Compared with \name, random noise consistently falls behind \name in its ability to reduce the generalization error on all tested tasks and is on average 0.7 points lower than \name (88.7 vs. 89.4). These results indicate that \name's approach of explicitly enforcing the smoothness of the loss landscape can result in better improvement.  

\textbf{Perturbations via ground-truth vs. label probability.} We also create one-hot labels and use those to generate perturbations instead of using label probability generated by the network. Table~\ref{tbl:other-perturbations} shows that using label probability (LP) consistently leads to higher accuracy than using the ground-truth (GT), e.g., 89.4 vs. 89.0 on average. Label probability leads to better generalization, probably because it provides a better measurement of the adversarial direction, which is the direction in the input space in which the label probability of the model is most sensitive to small perturbations. 

\textbf{Comparison with FreeLb.} We include a comparison between \name and FreeLb~\citep{freelb}. 
Although both FreeLb and \name achieve similar accuracy, \name is much faster than FreeLb (Table~\ref{tbl:comparison-with-freelb}). \name is faster because FreeLb is not explicitly designed for accelerating the adaptation speed and performs multiple ascent steps to calculate adversaries cross the full training process. In contrast, \name takes several optimizations to reduce the adversarial noise cost to enable efficient training with large batch sizes, which improves the overall computational efficiency.

\textbf{Scalability analysis varying GPUs.} 
Figure~\ref{fig:scalability-varying-gpus-sst} shows the scalability comparison on SST-2 after optimizations. While the speedup still plateaus at $4$ GPUs with a small batch size (e.g., $B=32$), the four large-batch configurations are able to scale well up to $32$ GPUs and take a similar amount of time with $32$ GPUs. \name scales better than \name without delaying PGA-$1$, and achieves a much faster training speed, especially in the 1-16 GPU range. 

\textbf{Train longer, generalize better?} Despite improved adaptation speed, one may still wonder whether simply performing large-batch adaptation longer would also close the generalization gap. 
Figure~\ref{fig:train-longer-mnli} shows the comparison between \name and the baseline on a batch size of 2K. \name obtains an accuracy of 85.2 after 6 epochs of training, whereas the baseline has difficulty to reach 84 after training twice longer (e.g., 12 epochs). \name achieves better accuracy because it explicitly penalizes model weights from getting stuck at sharp minima, leading to better generalizability. 

\textbf{Generalizability under different batch sizes.}
We also evaluate how different batch sizes affect the generalizability of adapting transformers. 
Figure~\ref{fig:varying-batch-comparison} shows the results on MNLI-m and SST-2. We make two major observations: (1) The accuracy tends to drop as the batch size increases. (2) While both the baseline and LAMB suffer from significant accuracy drop by drastically increasing the batch size (e.g., from 32 to 8K), \name is able to mitigate the generalization gap and consistently achieves higher accuracy than the baseline (e.g., 84.4 vs. 83.5 for MNLI, and 92.6 vs. 91.3 for SST-2 at batch size 8K) and LAMB (e.g., 84.4 vs. 83.9 for MNLI, and 92.6 vs. 91.7 for SST-2 at batch size 8K). These results indicate the benefit of \name is maintained by further increasing the batch size, which could bring even greater speedups when increasing the data parallelism degree.

\vspace{-3mm}
\section{Conclusions and Future Directions}
\label{sec:conclusion}
\vspace{-3mm}

In this paper, we study how to accelerate the adaptation speed of pre-trained Transformer models for NLU tasks. 
We introduce \name, an efficient large-batch adaptation method using carefully injected lightweight adversarial noises. 
The experiment results show that \name obtains up to 9.8$\times$ speedups on adapting transformer networks and outperforms state-of-the-art large-batch optimization methods in generalizability. Given the promising results of \name on accelerating the adaptation speed, it opens new research opportunities on applying \name to accelerate the more expensive pre-training tasks as well as emerging pre-trained transformer networks for computer vision domains tasks. 

\clearpage
\bibliography{references}
\bibliographystyle{icml2022}

%%%%%%%%%%%%%%%%%%%%%%%%%%%%%%%%%%%%%%%%%%%%%%%%%%%%%%%%%%%%%%%%%%%%%%%%%%%%%%%
%%%%%%%%%%%%%%%%%%%%%%%%%%%%%%%%%%%%%%%%%%%%%%%%%%%%%%%%%%%%%%%%%%%%%%%%%%%%%%%
% APPENDIX
%%%%%%%%%%%%%%%%%%%%%%%%%%%%%%%%%%%%%%%%%%%%%%%%%%%%%%%%%%%%%%%%%%%%%%%%%%%%%%%
%%%%%%%%%%%%%%%%%%%%%%%%%%%%%%%%%%%%%%%%%%%%%%%%%%%%%%%%%%%%%%%%%%%%%%%%%%%%%%%
% \newpage
\clearpage
\appendix

\section{Additional Results}
\label{sec:results}

In the part, we present results that are not included in the main text due to the space limit.

\subsection{The Usefulness of Adversarial Noises at Different Epochs}
\label{subsec:perturbation-usefulness}

\begin{figure}[h!]
 \centering
  \begin{minipage}[c]{0.4\textwidth}
   \centering
    \subfigure[MNLI-m]{\includegraphics[scale=0.5, keepaspectratio=true]{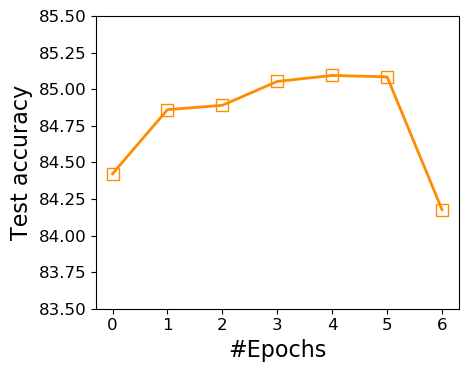}\label{fig:varying-switch-point-mnli}}
 \end{minipage}%
 \hfill
  \begin{minipage}[c]{0.4\textwidth}
   \centering
    \subfigure[SST-2]{\includegraphics[scale=0.5, keepaspectratio=true]{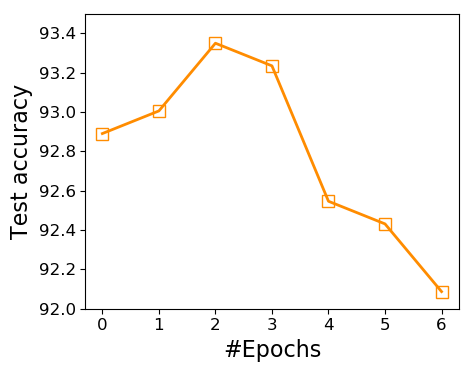}\label{fig:varying-switch-point-sst}}
 \end{minipage}%
 \caption{Accuracy results from delaying the injection of adversarial noises at different epochs.}
 \label{fig:varying-switch-point}
\end{figure}

%{subsec:inner-maximization}
In Section~\ref{sec:formulation}, we mention that no adversaries are needed at the initial epochs of adaptation.
To verify, we conduct experiments to measure the final accuracy corresponding to starting from regular training and switching to PGA-1 after $t_s$ epochs, where $t_s \in [T]$. Figure~\ref{fig:varying-switch-point} shows that enabling PGA-1 from the very beginning does not offer much improvement on accuracy. However, as we delay the injection of adversarial noises, the model accuracy starts to increase. 
% This is because, at initialization, the model’s parameters are relatively far from their final values and are less likely to get stuck at local minima. 
By delaying the injection of adversarial noises, we observe improved test accuracy on downstream tasks. However, it also seems that the adversarial noise should not be injected too late, which may inadvertently affect the accuracy. It is possible that a more advanced method to adaptively choose the value of $t_s$ is desired. However, given that (1) the primary focus of this work is to demonstrate that it is possible and effective to accelerate the adaptation of transformer networks via large-batch adaptation and adversarial noises and (2) the search space of is quite small for most downstream tasks, we leave this as an interesting research question for future exploration.

\section{Hyperparameters}
\label{sec:hyperparameters} 
For all configurations, we fine-tune against the GLUE datasets and set the maximum number of epochs to 6. We use a linear learning rate decay schedule with a warm-up ratio of 0.1. For \name, we set $\lambda=1$, perturbation clipping radius $\omega=10^{-5}$, step size $\rho=10^{-4}$, and $t_s$=\{3,5\}. These values worked well enough that we did not feel the need to explore more. For fairness, we perform a grid search of learning rates in the range of $\{$1e-5, 3e-5, 5e-5, 7e-5, 9e-5, 1e-4, 3e-4$\}$ for small batch sizes and $\{$5.6e-5, 8e-5, 1e-4, 1.7e-4, 2.4e-4, 2.8e-4, 4e-4, 5.6e-4, 1e-3$\}$ for large batch sizes. We keep the remaining hyperparameteres unchanged.

\section{Hyperparameter Tuning Cost for Large-Batch Adaptation with \name}
\label{sec:lr-scaling}

In this part, we investigate how large-batch adaptation affects the generalizability of transformer networks on downstream tasks. 
As there are various heuristics for setting the learning rates~\citep{do-not-decay-lr,train-imagenet-in-1hour,disciplined-hyperparameters,lamb},  and because few work studies the learning rate scaling effects on adapting pre-trained Transformer networks, we perform a grid search on learning rates $\{$1e-4, 3e-4,5e-4, 7e-4, 9e-4, 1e-3, 3e-3$\}$ and batch sizes $\{$1K, 2K, 4K, 8K$\}$ while keeping the other hyperparameters the same to investigate how \name affects the hyperparameter tuning effort. %\{1e-4, 3e-4,5e-4, 7e-4, 9e-4, 1e-3, 3e-3\}. and 
Table~\ref{tbl:lr-scaling-tuning} shows the results of using the square root scaling rule to decide the learning rates for large batch sizes vs. accuracy results with tuned learning rate results, without and with ScaLA. The first row represents the best accuracy found through fine-tuning with a small batch size 32. The next two rows correspond to fine-tuning with batch size 1024 using tuned learning rates vs. using the scaling rule. The last two rows represent fine-tuning using ScaLA with batch size 1024, also using tuned learning rates vs. the scaling rule. Even with square-root scaling, the large-batch baseline still cannot reach the small-batch accuracy (88.7 vs. 89.4). Moreover, although tuning the learning rates lead to better results on some datasets such as MNLI-m (84.9 vs. 85.1) and SST-2 (92.9 vs. 93.5), the square-root scaling rule leads to better results on other tasks such as QNLI (90.8 vs. 90) and QQP (91.4/88.4 vs. 90.9/87.7). So the best learning rates on fine-tuning tasks are not exactly sqrt. However, given that ScaLA with square-root learning rate scaling achieves on average better results than the grid search of learning rates (89.4 vs. 89.7), we suggest to use sqrt scaling for learning rates to simplify the hyperparameter tuning effort for ScaLA.

\onecolumn

\begin{table*}[!ht]
\caption{Evaluation results on hyperparameter tuning vs. using square-root learning rate scaling.}
\label{tbl:lr-scaling-tuning}
	\newcommand{\smallcolspc}{\hspace*{0.28em}}
	\newcommand{\colspc}{\hspace*{0.28em}}
    \small
    \renewcommand\sfsmaller{}
    \centering
\tabcolsep=0.10cm
\begin{tabular}{|@{\smallcolspc}l@{\smallcolspc}|@{\smallcolspc}c@{\smallcolspc}|@{\smallcolspc}c@{\smallcolspc}|@{\smallcolspc}c@{\smallcolspc}|@{\smallcolspc}c@{\smallcolspc}|@{\smallcolspc}c@{\smallcolspc}|}
\hline
{}      & 
\begin{tabular}[c]{@{}l@{}}MNLI-m \end{tabular}  &
\begin{tabular}[c]{@{}l@{}}QNLI \end{tabular} & 
\begin{tabular}[c]{@{}l@{}}QQP \end{tabular} &
\begin{tabular}[c]{@{}l@{}}SST-2 \end{tabular} & 
Avg \\ \hline
Bsz=32 (tuned, baseline)          & 84.8            & 90.6          & 91/88                 & 93.1          & 89.4         \\ \hline
Bsz=1024 (tuned, baseline)        & 84.3            & 89.3          & 89.6/86.1             & 93            & 88.5         \\ \hline
Bsz=1024 (scaling rule, baseline) & 83.9            & 89.2          & 90.6/87.4             & 92.5          & 88.7         \\ \hline
Bsz=1024 (tuned, ScaLA)           & 85.1            & 90            & 90.9/87.7             & 93.5          & 89.4         \\ \hline
Bsz=1024 (scaling rule, ScaLA)    & 84.9            & 90.8          & 91.4/88.4             & 92.9          & 89.7         \\ \hline
\end{tabular}
\end{table*}

% \clearpage

\section{Convergence Analysis}
\label{sec:proof}
In this section, we provide the formal statements and detailed proofs for the convergence rate. The convergence analysis builds on techniques and results in~\citep{davis2018stochastic, you2019large}. We consider the general problem of a two-player sequential game represented as nonconvex-nonconcave minimax optimization that is stochastic with respect to the outer (first) player playing $x \in \mathbb{X}$ while sampling $\xi$ from $Q$ and deterministic with respect to the inner (second) player playing $y \in \mathbb{Y}$, i.e.,
\begin{align}
\min_{x} \max_{y} \mathbb{E}_{\xi \sim Q}[f(x, y, \xi)] := \min_{x} \mathbb{E}_{\xi \sim Q}[g(x, \xi)]
\end{align}
%Recall the assumptions in the main text.
Since finding the Stackelberg equilibrium, i.e., the global solution to the saddle point problem, is NP-hard, we consider the optimality notion of a \textit{local minimax} point~\citep{jin2020local}. Since maximizing over $y$ may result in a non-smooth function even when $f$ is smooth, %, unlike non-convex optimization,
the norm of the gradient is not particularly a suitable metric to track the convergence progress of an iterative minimax optimization procedure. Hence, we use the gradient of the \textit{Moreau envelope}~\citep{davis2019stochastic} as the appropriate potential function. Let $\mu \in \mathbb{R}_+^h$. The $\mu$-Moreau envelope for a function $g: \mathbb{X} \to \mathbb{R}$ is defined as $g_{\mu}(x) := \min_z g(z) + \sum_{i=1}^h \frac{1}{2 \mu^i} \| x^i-z^i \|^2$. Another reason for the choice of this potential function is due to the special property~\citep{rockafellar2015convex} of the Moreau envelope that if its gradient $\nabla_x [g_{\mu}(x)]$ almost vanishes at $x$, such $x$ is close to a stationary point of the original function $g$. %Let $\mu \in \mathbb{R}_+^h$.

\textbf{Assumptions: }
We assume that $\mathbb{X} = \bigsqcup_{i=1}^{h} \mathbb{X}^i$ is partitioned into $h$ disjoint groups %such that $d_X = \sum_{i=1}^{h} d_{X}^i$
%$d_X$-dimensional 
, i.e., in terms of training a neural network, we can think of the network having the parameters partitioned into $h$ (hidden) layers. The measure $Q$ characterizes the training data. Let $\widehat{\nabla}_x f(x, y)$ denote the noisy estimate of the true gradient $\nabla_x f(x, y)$. We assume that the noisy gradients are unbiased, i.e., $\mathbb{E} [ \widehat{\nabla}_x f(x, y) ] = \nabla_x f(x, y)$. For each group $i \in [h]$, we make the standard (groupwise) boundedness assumption~\citep{ghadimi2013stochastic} on the variance of the stochastic gradients, i.e., $\mathbb{E} \|  \widehat{\nabla}^i_x f(x, y) - \nabla^i_x f(x, y) \|^2 \leq \sigma^2_i$, $\forall i \in [h]$. We assume that $f(x,y)$ has Lipschitz continuous gradients. Specifically, let $f(x,y)$ be $\alpha$-smooth in $x$ where $\alpha := (\alpha_1, \ldots, \alpha_h)$ denotes the $h$-dimensional vector of (groupwise) Lipschitz parameters, i.e., $\|  \nabla^i_x f(x_a, y) - \nabla^i_x f(x_b, y) \| \leq \alpha_i \| x_a^i - x_b^i\|$, $\forall i \in [h]$ and $x_a, x_b \in \mathbb{X}, y \in \mathbb{Y}$. Let $\kappa_\alpha := \frac{\max_i \alpha_i}{\min_i \alpha_i}$.

%There are $h$ hidden layers in the transformer network and each layer forms a group.
Super-scripts are used to index into a vector ($i$ denotes the group index and $j$ denotes an element in group $i$). For any $c \in \mathbb{R}$, the function $\nu: \mathbb{R} \rightarrow [\mathcal{L}, \mathcal{U}]$ clips its values, i.e., $\nu(c) := \max(\mathcal{L}, \min(c, \mathcal{U}))$ where $\mathcal{L} < \mathcal{U}$. Let $\|.\|$, $\|.\|_1$ and $\|.\|_\infty$ denote the $\ell_2$, $\ell_1$, and $\ell_\infty$ norms. We assume that the true gradients are bounded, i.e., $\| \nabla_x f(x,y) \|_\infty \leq \mathcal{G}$.

First, we begin with relevant supporting lemmas. The following lemma characterizes the convexity of an additive modification of $g$. % using its smoothness property.
\begin{lemma}[\citep{lin2020gradient, jin2020local, rafique2021weakly}]
\label{lem:weak_cvx}
Let $g(x) := \max_y f(x,y)$ with $f$ being $\alpha$-smooth in $x$ where $\alpha \in \mathbb{R}_+^h$ is the vector of groupwise Lipschitz parameters. Then, $g(x) + \sum_{i=1}^{h} \frac{\alpha_i}{2} \| x^i \|^2$ is convex in $x$.
\end{lemma}
%(details in Appendix~\ref{sec:proof})
The following property of the Moreau envelope relates it to the original function.
\begin{lemma}[\citep{rockafellar2015convex}]
\label{lem:original}
Let $g$ be defined as in Lemma~\ref{lem:weak_cvx}. Let $\widehat{x} = \arg \min_{\widetilde{x}} g(\widetilde{x}) + \sum_{i=1}^{h} \frac{1}{2\mu^i} \| \widetilde{x}^i-x^i \|^2$. Then, $\| g_\mu(x) \| \leq \epsilon$ implies $\| \widehat{x} - x \| \leq \| \mu \|_\infty \epsilon$ and $\min_h \| h \| \leq \epsilon$ with $h \in \partial g$ where $\partial g$ denotes the subdifferential of $g$.
\end{lemma}

We now present the formal version of Theorem~\ref{thm:oracle} in Theorem~\ref{thm:oracle_formal}. Note that Lemma~\ref{lem:original} facilitates giving the convergence guarantees in terms of the gradient of the Moreau envelope. Recall that $t \in [T]$ denotes the epochs corresponding to the outer maximization. Without loss of generality, we set the delay parameter for injection of the adversarial perturbation in Algorithm~\ref{algo:algorithm} as $t_s = 0$. Here, we assume that the PGA provides an $\epsilon$-approximate maximizer.
\begin{theorem}[Groupwise outer minimization with an $\epsilon$-approximate inner maximization oracle]
\label{thm:oracle_formal}
Let us define relevant constants as $\mathcal{D} := \left( g_{1/2 \alpha}(x_{0}) -\mathbb{E} (\min_x g(x)) \right)$ being the optimality gap due to initialization, $\kappa_\alpha := \frac{\max_i \alpha_i}{\min_i \alpha_i}$ being the condition number, $\| \nabla_x f(x,y) \|_\infty \leq \mathcal{G}$ being gradient bound, $\mathcal{Z} := \max_{i, j, t} \frac{( \widehat{x}^{i,j}_t - x^{i,j}_t )}{(\nabla^{i,j}_t)} \sigma_i$ being the variance term, $\mathcal{L, U}$ being clipping constants such that $\mathcal{L} \leq \mathcal{U}$. For the outer optimization, setting the learning rate as $\eta = \frac{1}{\mathcal{U}\sqrt{T}}$ and scaling batch size as $b = \frac{16 T \mathcal{L}^2 \mathcal{Z}^2}{\mathcal{U}^2}$, we have
%\begin{align}
%    \frac{1}{T} \sum_{t=0}^{T-1} \mathbb{E} \left( \sum_{i=1}^{h} \frac{1}{4\alpha_i} \| \nabla g_{1/2 \alpha}(\widehat{x}^i_t) \|^2 \right) & \leq \epsilon + \frac{\mathcal{D}}{2 \zeta \sqrt{T}}
%\end{align}
% \frac{1}{T} \sum_{t=0}^{T-1} 
\begin{align}
    \mathbb{E} \left[ \| \nabla g_{1/2 \alpha}(\overline{x}) \|^2 \right] & \leq 4 \epsilon \| \alpha \|_\infty  + \frac{2 \kappa_\alpha \mathcal{D G}}{\sqrt{T}}
\end{align}
%Let the learning rate $\eta_t$ be $\frac{\eta_0}{\sqrt{T+1}}$. Then,
% \begin{align*}
% \left( \mathbb{E} \left[ \frac{1}{\sqrt{h}}\sum_{i=1}^{h} \| \nabla^i g_{1/2 \alpha}(\widehat{x}) \| \right] \right)^2 \leq 2 \frac{g_{1/2 \alpha}(x_0)- g(x^*) + \alpha \eta_0^2 \sum_{i=1}^{h} \beta_i^2}{\eta_0 / \sqrt{T+1}} + 4 \alpha \epsilon
% \end{align*}
where $\overline{x}$ is the estimator obtained from running $T$ steps of Algorithm~\ref{algo:algorithm} and picking $x_t$ uniformly at random for $t \in [T]$.
\end{theorem}

\begin{proof}
In this proof, for brevity, we define the vector $\nabla_t := \nabla_x f(x, y)$, i.e., the gradient of the objective with respect to $x$, evaluated at the outer step $t$. % $(x_t, y_t)$;
Since evaluating gradients using mini-batches produces noisy gradients, we use $\widehat{\nabla}$ to denote the noisy version of a true gradient $\nabla$, i.e., $\widehat{\nabla} = \nabla + \Delta$ for a noise vector $\Delta$. For any outer step $t$, we have $f(x_t, \widehat{y}) \geq g(x_t) - \epsilon$ where $\widehat{y}$ is an $\epsilon$-approximate maximizer. For any $\widetilde{x} \in \mathbb{X}$, using the smoothness property (Lipschitz gradient) of $f$, we have
\begin{align}
g(\widetilde{x}) &\geq f(\widetilde{x}, y_t) \nonumber \\
&\geq f(x_t, y_t) + \sum_{i=1}^{h} \langle \nabla^i_t, \widetilde{x}^i-x^i_t \rangle - \sum_{i=1}^{h} \frac{\alpha_i}{2} \| \widetilde{x}^i - x^i_t \|^2 \nonumber \\
&\geq g(x_t) - \epsilon + \sum_{i=1}^{h} \langle \nabla^i_t, \widetilde{x}^i-x^i_t \rangle - \sum_{i=1}^{h} \frac{\alpha_i}{2} \| \widetilde{x}^i - x^i_t \|^2 \label{eqn:smoothx}
\end{align}
%&\geq g(x_t) - \epsilon + \sum_{i=1}^{h} \langle \nabla^i_x f(x^i_t, y_t), \widetilde{x}^i-x^i_t \rangle - \sum_{i=1}^{h} \frac{\alpha_i}{2} \| \widetilde{x}^i - x^i_t \|^2

%Let $\mu \in \mathbb{R}^h$.
Let $\phi_\mu(x, z) := g(z) + \sum_{i=1}^h \frac{1}{2 \mu^i} \| x^i-z^i \|^2$. Recall that the $\mu$-Moreau envelope for $g$ is defined as $g_{\mu}(x) := \min_z \phi_\mu(x, z)$ and its gradient is the groupwise proximal operator given by $\nabla_x [g_{\mu}(x)] =  \left[ \frac{1}{\mu^1} \left( x^1 - \arg \min_{z^1} \phi_\mu(x, z) \right), \ldots,  \frac{1}{\mu^h} \left( x^h - \arg \min_{z^h} \phi_\mu(x, z) \right) \right]$. 
%Recall that, for some $z$, the $\mu$-Moreau envelope is $g_{\mu}(x) := \min_z \left( g(z) + \frac{1}{2 \mu} \| x-z \|^2 \right)$ and its gradient is given by $\nabla_x [g_{\mu}(x)] = \frac{1}{\mu} \left( x - \min_z \left( g(z) + \frac{1}{2 \mu} \| x-z \|^2 \right) \right)$. 

%where we have use the $\epsilon$-approximate oracle.
Now, let $\widehat{x}_t = \arg \min_x \phi_{1/2\alpha}(x_t, x) = \arg \min_x \left( g(x) + \sum_{i=1}^{h} \alpha_i \|x^i_t-x^i\|^2 \right)$.
Then, plugging in the update rule for $x$ at step $t+1$ in terms of quantities at step $t$, using the shorthand $\nu^i_t := \nu(\| x^i_t \|)$ and conditioning on the filtration up to time $t$, we have
\begin{align}
    g_{1/2 \alpha}(x_{t+1}) &\leq g(\widehat{x}_t) + \sum_{i=1}^{h} \alpha_i \| x^i_{t+1} - \widehat{x}^i_t \|^2 \nonumber \\
    &\leq g(\widehat{x}_t) + \sum_{i=1}^{h} \alpha_i \left\| x^i_t - \eta_t \nu^i_t \frac{\widehat{\nabla}^i_t}{\| \widehat{\nabla}^i_t \|} - \widehat{x}^i_t \right\|^2 \nonumber \\
    &\leq g(\widehat{x}_t) + \sum_{i=1}^{h} \alpha_i \left\| x^i_t - \widehat{x}^i_t \right\|^2 + \sum_{i=1}^{h} 2 \alpha_i \eta_t \left\langle \nu^i_t \frac{\widehat{\nabla}^i_t}{\| \widehat{\nabla}^i_t \|}, \widehat{x}^i_t - x^i_t \right\rangle + \sum_{i=1}^{h} \alpha_i \eta_t^2 (\nu^i_t)^2 \nonumber \\
    &\leq g_{1/2 \alpha}(x_{t}) + \sum_{i=1}^{h} 2 \alpha_i \eta_t \left\langle \nu^i_t \frac{\widehat{\nabla}^i_t}{\| \widehat{\nabla}^i_t \|}, \widehat{x}^i_t - x^i_t \right\rangle + \sum_{i=1}^{h} \alpha_i \eta_t^2 (\nu^i_t)^2 \nonumber \\
    &\leq g_{1/2 \alpha}(x_{t}) + 2 \eta_t \sum_{i=1}^{h} \alpha_i \nu^i_t \sum_{j=1}^{d_i} \left( \frac{\widehat{\nabla}^{i,j}_t}{\| \widehat{\nabla}^i_t \|} - \frac{{\nabla^{i,j}_t}}{\| {\nabla^i_t} \|} + \frac{{\nabla^{i,j}_t}}{\| {\nabla^i_t} \|} \right) \times ( \widehat{x}^{i,j}_t - x^{i,j}_t ) + \sum_{i=1}^{h} \alpha_i \eta_t^2 (\nu^i_t)^2 \nonumber \\
    &\leq g_{1/2 \alpha}(x_{t}) + 2 \eta_t \sum_{i=1}^{h} \alpha_i \nu^i_t \sum_{j=1}^{d_i} \left( \frac{{\nabla^{i,j}_t}}{\| {\nabla^i_t} \|} \right) \times ( \widehat{x}^{i,j}_t - x^{i,j}_t ) \nonumber \\
    &+ 2 \eta_t \sum_{i=1}^{h} \alpha_i \nu^i_t \sum_{j=1}^{d_i} \left( \frac{\widehat{\nabla}^{i,j}_t}{\| \widehat{\nabla}^i_t \|} - \frac{{\nabla^{i,j}_t}}{\| {\nabla^i_t} \|} \right) \times ( \widehat{x}^{i,j}_t - x^{i,j}_t ) + \sum_{i=1}^{h} \alpha_i \eta_t^2 (\nu^i_t)^2 \nonumber \\
    &\leq g_{1/2 \alpha}(x_{t}) + 2 \eta_t \sum_{i=1}^{h} \frac{\alpha_i \nu^i_t}{\| {\nabla^i_t} \|} \left\langle {\nabla^i_t}, \widehat{x}^i_t - x^i_t \right\rangle \nonumber \\
    &+ 2 \eta_t \sum_{i=1}^{h} \alpha_i \nu^i_t \sum_{j=1}^{d_i} \left( \frac{\nabla^{i,j}_t + \Delta^{i,j}_t}{\| \nabla^i_t + \Delta^i_t \|} - \frac{{\nabla^{i,j}_t}}{\| {\nabla^i_t} \|} \right) \times ( \widehat{x}^{i,j}_t - x^{i,j}_t ) + \sum_{i=1}^{h} \alpha_i \eta_t^2 (\nu^i_t)^2 \nonumber % \\
\end{align}
\begin{align}
    &\leq g_{1/2 \alpha}(x_{t}) + 2 \eta_t \mathcal{U} \sum_{i=1}^{h} \frac{\alpha_i}{\| {\nabla^i_t} \|} \left\langle {\nabla^i_t}, \widehat{x}^i_t - x^i_t \right\rangle \nonumber \\
    &+ 2 \eta_t \sum_{i=1}^{h} \alpha_i \nu^i_t \sum_{j=1}^{d_i} \left( \frac{ \| {\nabla^i_t} \| (\nabla^{i,j}_t) (\nabla^{i,j}_t + \Delta^{i,j}_t) - \| \nabla^i_t + \Delta^i_t \| (\nabla^{i,j}_t)^2}{\| \nabla^i_t + \Delta^i_t \| \| {\nabla^i_t} \|} \right) \times \frac{( \widehat{x}^{i,j}_t - x^{i,j}_t )}{(\nabla^{i,j}_t)} \nonumber \\
    &+ \sum_{i=1}^{h} \alpha_i \eta_t^2 (\nu^i_t)^2 \nonumber \\
    &\stackrel{E_1}{\leq} g_{1/2 \alpha}(x_{t}) + 2 \eta_t \mathcal{U} \max_i \frac{\alpha_i}{\| {\nabla^i_t} \|} \left( g(\widehat{x}_t) - g(x_t) + \epsilon + \sum_{i=1}^{h} \frac{\alpha_i}{2} \| \widehat{x}^i - x^i_t \|^2 \right) \nonumber \\
    &+ 2 \eta_t \sum_{i=1}^{h} \alpha_i \nu^i_t \max_j \frac{( \widehat{x}^{i,j}_t - x^{i,j}_t )}{(\nabla^{i,j}_t)} \left( \frac{ \langle \nabla^{i}_t, \nabla^{i}_t + \Delta^{i}_t \rangle - \| \nabla^i_t + \Delta^i_t \| \| \nabla^i_t \|}{\| \nabla^i_t + \Delta^i_t \|} \right) + \sum_{i=1}^{h} \alpha_i \eta_t^2 (\nu^i_t)^2 \nonumber \\
    &\leq g_{1/2 \alpha}(x_{t}) + 2 \eta_t \mathcal{U} \max_i \frac{\alpha_i}{\| {\nabla^i_t} \|} \left( g(\widehat{x}_t) - g(x_t) + \epsilon + \sum_{i=1}^{h} \frac{\alpha_i}{2} \| \widehat{x}^i - x^i_t \|^2 \right) \nonumber \\
    &- 2 \eta_t \sum_{i=1}^{h} \alpha_i \nu^i_t \max_j \frac{( \widehat{x}^{i,j}_t - x^{i,j}_t )}{(\nabla^{i,j}_t)} \left( \frac{ \| \nabla^i_t + \Delta^i_t \| \| \nabla^i_t \| - \| \nabla^i_t + \Delta^i_t \|^2 + \langle \Delta^{i}_t, \nabla^{i}_t + \Delta^{i}_t \rangle}{\| \nabla^i_t + \Delta^i_t \|} \right) \nonumber \\
    &+ \sum_{i=1}^{h} \alpha_i \eta_t^2 (\nu^i_t)^2  \\ %\nonumber
    &\leq g_{1/2 \alpha}(x_{t}) + 2 \eta_t \mathcal{U} \max_i \frac{\alpha_i}{\| {\nabla^i_t} \|} \left( g(\widehat{x}_t) - g(x_t) + \epsilon + \sum_{i=1}^{h} \frac{\alpha_i}{2} \| \widehat{x}^i - x^i_t \|^2 \right) \nonumber \\
    &- 2 \eta_t \sum_{i=1}^{h} \alpha_i \nu^i_t \max_j \frac{( \widehat{x}^{i,j}_t - x^{i,j}_t )}{(\nabla^{i,j}_t)} \left( \| \nabla^i_t \| - \| \nabla^i_t + \Delta^i_t \| - \frac{ | \langle \Delta^{i}_t, \nabla^{i}_t + \Delta^{i}_t \rangle | }{\| \nabla^i_t + \Delta^i_t \|} \right) + \sum_{i=1}^{h} \alpha_i \eta_t^2 (\nu^i_t)^2 \nonumber \\
    &\stackrel{E_2}{\leq} g_{1/2 \alpha}(x_{t}) + 2 \eta_t \mathcal{U} \max_i \frac{\alpha_i}{\| {\nabla^i_t} \|} \left( g(\widehat{x}_t) - g(x_t) + \epsilon + \sum_{i=1}^{h} \frac{\alpha_i}{2} \| \widehat{x}^i - x^i_t \|^2 \right) \nonumber \\
    &- 2 \eta_t \sum_{i=1}^{h} \alpha_i \nu^i_t \max_j \frac{( \widehat{x}^{i,j}_t - x^{i,j}_t )}{(\nabla^{i,j}_t)} \left( \| \nabla^i_t \| - \| \nabla^i_t + \Delta^i_t \| - \| \Delta^{i}_t \| \right) + \sum_{i=1}^{h} \alpha_i \eta_t^2 (\nu^i_t)^2 \nonumber \\
    &\stackrel{E_3}{\leq} g_{1/2 \alpha}(x_{t}) + 2 \eta_t \mathcal{U} \max_i \frac{\alpha_i}{\| {\nabla^i_t} \|} \left( g(\widehat{x}_t) - g(x_t) + \epsilon + \sum_{i=1}^{h} \frac{\alpha_i}{2} \| \widehat{x}^i - x^i_t \|^2 \right) \nonumber \\
    &- 4 \eta_t \sum_{i=1}^{h} \alpha_i \nu^i_t \max_j \frac{( \widehat{x}^{i,j}_t - x^{i,j}_t )}{(\nabla^{i,j}_t)} \| \Delta^{i}_t \| + \sum_{i=1}^{h} \alpha_i \eta_t^2 (\nu^i_t)^2 \nonumber \\
    g_{1/2 \alpha}(x_{T}) & \stackrel{E_4}{\leq} g_{1/2 \alpha}(x_{0}) + 2 \mathcal{U} \sum_{t=0}^{T-1} \eta_t \max_i \frac{\alpha_i}{\| {\nabla^i_t} \|} \left( g(\widehat{x}_t) - g(x_t) + \epsilon + \sum_{i=1}^{h} \frac{\alpha_i}{2} \| \widehat{x}^i - x^i_t \|^2 \right) \nonumber \\
    &- 4 \sum_{t=0}^{T-1} \eta_t \sum_{i=1}^{h} \alpha_i \nu^i_t \max_j \frac{( \widehat{x}^{i,j}_t - x^{i,j}_t )}{(\nabla^{i,j}_t)} \| \Delta^{i}_t \| + \sum_{t=0}^{T-1} \sum_{i=1}^{h} \alpha_i \eta_t^2 (\nu^i_t)^2 \nonumber
\end{align}
where we have used H\"{o}lder's inequality along with bound~\eqref{eqn:smoothx} in $E_1$, Cauchy-Schwarz inequality in $E_2$, triangle inequality in $E_3$, telescoping sum in $E_4$. Rearranging and using $\eta_t = \eta$ in $E_5$ along with H\"{o}lder's inequality,
\begin{align}
    \frac{1}{2 \eta \mathcal{U}} \left( g_{1/2 \alpha}(x_{T}) - g_{1/2 \alpha}(x_{0}) \right) & \leq \sum_{t=0}^{T-1} \max_i \frac{\alpha_i}{\| {\nabla^i_t} \|} \left( g(\widehat{x}_t) - g(x_t) + \epsilon + \sum_{i=1}^{h} \frac{\alpha_i}{2} \| \widehat{x}^i - x^i_t \|^2 \right) \nonumber \\
    &- \frac{2}{\mathcal{U}} \sum_{t=0}^{T-1} \sum_{i=1}^{h} \alpha_i \nu^i_t \max_j \frac{( \widehat{x}^{i,j}_t - x^{i,j}_t )}{(\nabla^{i,j}_t)} \| \Delta^{i}_t \| + \frac{\eta}{2 \mathcal{U}} \sum_{t=0}^{T-1} \sum_{i=1}^{h} \alpha_i (\nu^i_t)^2 \nonumber \\
    \frac{1}{2 \eta \mathcal{U}} \left( g_{1/2 \alpha}(x_{T}) - g_{1/2 \alpha}(x_{0}) \right) & \stackrel{E_5}{\leq} \max_{i,t} \frac{\alpha_i}{\| {\nabla^i_t} \|} \sum_{t=0}^{T-1} \left( g(\widehat{x}_t) - g(x_t) + \epsilon + \sum_{i=1}^{h} \frac{\alpha_i}{2} \| \widehat{x}^i - x^i_t \|^2 \right) \nonumber \\
    &- \frac{2}{\mathcal{U}} \sum_{t=0}^{T-1} \sum_{i=1}^{h} \alpha_i \nu^i_t \max_j \frac{( \widehat{x}^{i,j}_t - x^{i,j}_t )}{(\nabla^{i,j}_t)} \| \Delta^{i}_t \| + \frac{\eta}{2 \mathcal{U}} \sum_{t=0}^{T-1} \sum_{i=1}^{h} \alpha_i (\nu^i_t)^2 \nonumber
\end{align}
%Taking expectation on both sides and rearranging,
Dividing by $T$ and rearranging,
\begin{align}
    \frac{1}{T} \sum_{t=0}^{T-1} \left( g(x_t) - g(\widehat{x}_t) - \sum_{i=1}^{h} \frac{\alpha_i}{2} \| \widehat{x}^i - x^i_t \|^2 \right) & \leq \epsilon - \frac{1}{2 \eta \mathcal{U} \zeta T} \left( g_{1/2 \alpha}(x_{T}) - g_{1/2 \alpha}(x_{0}) \right) \nonumber \\
    &- \frac{2}{\mathcal{U} \zeta T} \sum_{t=0}^{T-1} \sum_{i=1}^{h} \alpha_i \nu^i_t \max_j \frac{( \widehat{x}^{i,j}_t - x^{i,j}_t )}{(\nabla^{i,j}_t)} \| \Delta^{i}_t \| \nonumber \\
    &+ \frac{\eta}{2 \mathcal{U} \zeta T} \sum_{i=1}^{h} \alpha_i \sum_{t=0}^{T-1} (\nu^i_t)^2 \nonumber
\end{align}
where we define $\zeta:=\max_{i,t} \frac{\alpha_i}{\| {\nabla^i_t} \|}$.
%Let $\mathcal{F}_t$ be the most recent element of the underlying filtration. Taking expectation conditioned on this (ie, the latest $\sigma$-algebra generated using all information up to and including the events at time $t$).
Defining $\mathcal{D} := \left( g_{1/2 \alpha}(x_{0}) - \mathbb{E} (\min_x g(x)) \right)$ and taking expectation with respect to $\xi$ on both sides, we have
\begin{align}
    \frac{1}{T} \sum_{t=0}^{T-1} \mathbb{E} \left( g(x_t) - g(\widehat{x}_t) - \sum_{i=1}^{h} \frac{\alpha_i}{2} \| \widehat{x}^i - x^i_t \|^2 \right) & \leq \epsilon + \frac{\mathcal{D}}{2 \eta \mathcal{U} \zeta T} \nonumber \\
    &- \frac{2 \mathcal{L}}{\mathcal{U} \zeta T} \sum_{t=0}^{T-1} \sum_{i=1}^{h} \alpha_i \max_j \frac{( \widehat{x}^{i,j}_t - x^{i,j}_t )}{(\nabla^{i,j}_t)} \mathbb{E} \| \Delta^{i}_t \| + \frac{\eta \mathcal{U} \| \alpha \|_1}{2 \zeta}  \nonumber \\
    & \stackrel{E_6}{\leq} \epsilon + \frac{\mathcal{D}}{2 \eta \mathcal{U} \zeta T} \nonumber \\
    &- \frac{2 \mathcal{L}}{\mathcal{U} \zeta T} \sum_{t=0}^{T-1} \sum_{i=1}^{h} \alpha_i \max_j \frac{( \widehat{x}^{i,j}_t - x^{i,j}_t )}{(\nabla^{i,j}_t)} \frac{\sigma_i}{\sqrt{b}} + \frac{\eta \mathcal{U} \| \alpha \|_1}{2 \zeta}  \nonumber \\
    & \stackrel{E_7}{\leq} \epsilon + \frac{\mathcal{D}}{2 \eta \mathcal{U} \zeta T} \nonumber \\
    &- \frac{2 \mathcal{L} \| \alpha \|_1}{\mathcal{U} \zeta \sqrt{b}} \max_{i, j, t} \frac{( \widehat{x}^{i,j}_t - x^{i,j}_t )}{(\nabla^{i,j}_t)} \sigma_i + \frac{\eta \mathcal{U} \| \alpha \|_1}{2 \zeta}  \nonumber \\
    & \stackrel{E_8}{=} \epsilon + \frac{\mathcal{D}}{2 \eta \mathcal{U} \zeta T} - \frac{2 \mathcal{L} \| \alpha \|_1 \mathcal{Z}}{\mathcal{U} \zeta \sqrt{b}} + \frac{\eta \mathcal{U} \| \alpha \|_1}{2 \zeta} \label{eqn:rhs}
\end{align}
where we have used the assumption on the variance of stochastic gradients 
%, ie, $\mathbb{E} \| \nabla^i - \widehat{\nabla}^i \|^2 \leq \sigma_i$, $\forall i \in [h]$
in $E_6$, H\"{o}lder's inequality in $E_7$ and we define $\mathcal{Z} := \max_{i, j, t} \frac{( \widehat{x}^{i,j}_t - x^{i,j}_t )}{(\nabla^{i,j}_t)} \sigma_i$ in $E_8$; $b$ denotes batch size. Now, we lower bound the left hand side using the convexity of the additive modification of $g$. % using the property that, for any $x$, $g(x) + \sum_{i=1}^h \alpha_i \| x^i - x^i_t\|^2$ is $\alpha$-strongly convex.
\begin{align}
    & g(x_t) - g(\widehat{x}_t) - \sum_{i=1}^{h} \frac{\alpha_i}{2} \| \widehat{x}^i - x^i_t \|^2 \nonumber \\
    &\geq g(x_t) + 0 - g(\widehat{x}_t) - \sum_{i=1}^{h} \alpha_i \| \widehat{x}^i - x^i_t \|^2 + \sum_{i=1}^{h} \frac{\alpha_i}{2} \| \widehat{x}^i - x^i_t \|^2 \nonumber \\
    &\geq \left( \left( g(x_t) + \sum_{i=1}^{h} \alpha_i \| x^i_t - x^i_t \|^2 \right) - \min_x \left( g(x_t) + \sum_{i=1}^{h} \alpha_i \| x^i - x^i_t \|^2 \right)  \right) + \sum_{i=1}^{h} \frac{\alpha_i}{2} \| \widehat{x}^i - x^i_t \|^2 \nonumber \\
    &\geq \sum_{i=1}^{h} \frac{\alpha_i}{2} \| \widehat{x}^i - x^i_t \|^2 + \sum_{i=1}^{h} \frac{\alpha_i}{2} \| \widehat{x}^i - x^i_t \|^2 {=} \sum_{i=1}^{h} \frac{4\alpha_i^2}{4\alpha_i} \| \widehat{x}^i - x^i_t \|^2 \nonumber \\
    & \stackrel{E_9}{\geq} \frac{1}{4 \max_i \alpha_i } \| \nabla g_{1/2 \alpha}(x_t) \|^2 \label{eqn:lhs}
\end{align}
where we have used the expression for the gradient of the Moreau envelope in $E_9$. Combining the inequalities from Equation~\eqref{eqn:lhs} and Equation~\eqref{eqn:rhs}, we have
\begin{align}
    \frac{1}{T} \sum_{t=0}^{T-1} \mathbb{E} \left( \frac{1}{4 \max_i \alpha_i } \| \nabla g_{1/2 \alpha}(x_t) \|^2 \right) & \leq \epsilon + \frac{\mathcal{D}}{2 \eta \mathcal{U} \zeta T} + \left( \frac{\eta \mathcal{U}}{2 \zeta} - \frac{2 \mathcal{L} \mathcal{Z}}{\mathcal{U} \zeta \sqrt{b}} \right) \| \alpha \|_1 \nonumber
\end{align}
Setting the learning rate as $\eta = \frac{1}{\mathcal{U}\sqrt{T}}$ and batch size as $b = \frac{16 T \mathcal{L}^2 \mathcal{Z}^2}{\mathcal{U}^2}$,
\begin{align}
    \frac{1}{T} \sum_{t=0}^{T-1} \mathbb{E} \left[ \| \nabla g_{1/2 \alpha}(x_t) \|^2 \right] & \leq 4 \epsilon \max_i \alpha_i  + \frac{2 \mathcal{D} \max_i \alpha_i}{\zeta \sqrt{T}} \nonumber
\end{align}
Now, to simplify $\zeta$, using the inequality that $\max_k (a_k \cdot b_k) \geq \min_{k_a} a_{k_a} \cdot  \min_{k_b} b_{k_b}$ for two finite sequences $\{ a, b \}$ with positive values, along with the bounded gradients assumption, we have
\begin{align}
    \frac{1}{T} \sum_{t=0}^{T-1} \mathbb{E} \left[ \| \nabla g_{1/2 \alpha}(x_t) \|^2 \right] & \leq 4 \epsilon \max_i \alpha_i  + \frac{2 \mathcal{D G} \max_i \alpha_i}{\sqrt{T} \min_i \alpha_i} = 4 \epsilon \| \alpha \|_\infty  + \frac{2 \kappa_\alpha \mathcal{D G}}{\sqrt{T}} \nonumber
\end{align}
where $\kappa_\alpha := \frac{\max_i \alpha_i}{\min_i \alpha_i}$.
%\begin{align}
%    \frac{1}{T} \sum_{t=0}^{T-1} \mathbb{E} \left( \sum_{i=1}^{h} \frac{1}{4\alpha_i} \| \nabla g_{1/2 \alpha}(x^i_t) \|^2 \right) & \leq \epsilon + \frac{\mathcal{D}}{2 \zeta \sqrt{T}} \nonumber
%\end{align}
% + \left( \frac{1}{2 \zeta \sqrt{T}} - \frac{1}{2 \zeta \sqrt{T}} \right) \| \alpha \|_1
\end{proof}
%Next, we design the optimal regularization schedule.

%\begin{comment}
In analyzing inexact version, as in Theorem~\ref{thm:oracle_formal}, we assumed the availability of an adversarial oracle. Next, we open up the adversarial oracle to characterize the oracle-free complexity. In order to do this, we will assume additional properties about the function $f$ as well as our deterministic perturbation space, $\mathbb{Y}_t \subseteq \mathbb{Y}$, $\forall t \in [T]$. Note that, for any given $t$, $y_\tau \in \mathbb{Y}_t$, $\forall \tau \in \mathcal{T}$. We recall the following guarantee for generalized non-convex projected gradient ascent.
\begin{lemma}[\citep{jain2017non}]
\label{lem:pgd}
For every $t$, Let $f(x_t,\cdot)$ satisfy restricted strong convexity with parameter $\mathcal{C}$ and restricted strong smoothness with parameter $\mathcal{S}$ over a non-convex constraint set with $\mathcal{S}/\mathcal{C} < 2$, ie, $\frac{\mathcal{C}}{2} \| z-y \|^2 \leq f(x_t,y) - f(x_t,z) - \langle \nabla_z f(x_t,z), y-z \rangle \leq \frac{\mathcal{S}}{2} \| z-y \|^2$ for $ y, z \in \mathbb{Y}_t$.
For any given $t$, let the PGA-$\mathcal{T}$ algorithm $y_{\tau} \leftarrow \Pi_\epsilon [ y_{\tau-1} + \rho \nabla_{y} f(x_t,y) ]$ be executed with step size $\rho = 1/\mathcal{S}$. Then after at most $\mathcal{T} = O\left( \frac{\mathcal{C}}{2 \mathcal{C} - \mathcal{S}} \log \frac{1}{\epsilon} \right)$ steps, $f(x_t,y_{\mathcal{T}}) \geq \max_y f(x_t,y) - \epsilon$.
\end{lemma}
%\begin{proof}
%$\frac{\mathcal{C}}{2} \| z-y \|^2 \leq f(x_t,y) - f(x_t,z) - \langle \nabla_z f(x_t,z), y-z \rangle \leq \frac{\mathcal{S}}{2} \| z-y \|^2$
%
%$\frac{\mathcal{C}}{2} \| z-y \|^2 \leq f(x_{t+1},y) - f(x_{t+1},z) - \langle \nabla_z f(x_{t+1},z), y-z \rangle \leq \frac{\mathcal{S}}{2} \| z-y \|^2$
%\end{proof}
Using Theorem~\ref{thm:oracle_formal} and Lemma~\ref{lem:pgd} (together with the additional restricted strong convexity/smoothness assumptions), we have the following theorem on the full oracle-free rates for Algorithm~\ref{algo:algorithm}.
\begin{theorem}[Groupwise outer minimization with inner maximization using projected gradient ascent]
Setting the inner iteration count as $\mathcal{T} = O\left( \frac{\mathcal{C}}{2 \mathcal{C} - \mathcal{S}} \log \frac{8 \| \alpha \|_\infty}{\epsilon} \right)$ and the outer iteration count as $T = \frac{16 \kappa_\alpha \mathcal{D}^2 \mathcal{G}^2}{\epsilon^2}$, for a combined total of $O(\frac{1}{\epsilon^2} \log \frac{1}{\epsilon})$ adaptive adversarial iterations, Algorithm~\ref{algo:algorithm} achieves $\mathbb{E} \left[ \| \nabla g_{1/2 \alpha}(\overline{x}) \|^2 \right] \leq \epsilon$.
%$\widetilde{O}(\kappa_\alpha^2 \epsilon^{-2})$ iterations
\end{theorem}
%\end{comment}

%Next, with momentum and LAMB.
%Variance reduction
%Regularization and convex rates
%Importance sampling
%Dual
%$Q_t$
%$\lambda_t$
%$t > T$ stop
%PGD(y)
%truncated Q
%We first consider the exact case
%\begin{equation}
%\label{eqn:dap}
%    \min_{w} \big( {\mathbb{E}}_{x \sim P}[\ell(w, x)] + \lambda %{\mathbb{E}}_{x \sim Q}[\phi(w,x)] \big)
%\end{equation}
%Analyze LARS with the inexact adversarial (max) oracle.
%\begin{lemma}
%${\mathbb{E}}_{x \sim P} [x] = p \implies {\mathbb{E}}_{x \sim Q} %[x] = {\mathbb{E}}_{x \sim P}[\nu x]$ where $\nu = q(x)/p(x)$
%\end{lemma}

% \onecolumn

\end{document}